\documentclass{article}

\PassOptionsToPackage{authoryear, sort&compress, round}{natbib}

\usepackage[preprint]{neurips_2024}

\usepackage[utf8]{inputenc} %
\usepackage[T1]{fontenc}    %
\usepackage{hyperref}       %
\usepackage{url}            %
\usepackage{booktabs}       %
\usepackage{amsfonts}       %
\usepackage{nicefrac}       %
\usepackage{microtype}      %
\usepackage{xcolor,colortbl}%
\definecolor{aliceblue}{rgb}{0.94, 0.97, 1.0}
\definecolor{lavender}{rgb}{0.9, 0.9, 0.98}

\usepackage{adjustbox}
\usepackage{subfigure}
\definecolor{dred}{RGB}{153,80,43}
\definecolor{dblue}{RGB}{0,114,178}
\definecolor{lightgray}{gray}{0.75}
\hypersetup{
    colorlinks=true,
    breaklinks=true,
    urlcolor=dblue, %
    linkcolor=dred, %
    citecolor=dblue
}
\usepackage{arydshln}
\usepackage{tcolorbox}

\usepackage{amsmath}
\usepackage{amssymb}
\usepackage{amsthm}
\usepackage{mathtools}
\usepackage{thmtools}
\usepackage{thm-restate}
\usepackage[capitalise]{cleveref}
\crefname{reassumption}{Assumption}{Assumptions}
\crefname{relemma}{Lemma}{Lemmas}
\crefname{retheorem}{Theorem}{Theorems}
\crefname{reremark}{Remark}{Remarks}
\crefformat{equation}{Eq. #2#1#3}
\crefformat{section}{Sec. #2#1#3}

\DeclareMathOperator*{\argmax}{arg\,max}
\DeclareMathOperator*{\argmin}{arg\,min}
\DeclareMathOperator{\sg}{sg}

\DeclareMathOperator{\mtrace}{Tr}
\newcommand{\real}{\mathbb{R}}
\usepackage{mathrsfs}
\newcommand{\probspace}{\mathscr{P}}
\newcommand{\E}{\mathbb{E}}
\newcommand{\diag}{\mbox{diag}}
\newcommand{\PhiVARAC}{\Phi_\text{var}}
\newcommand{\PhiAC}{\Phi_\text{ac}}

\title{A Unifying Framework for Action-Conditional Self-Predictive Reinforcement Learning}

\author{%
  Khimya Khetarpal\thanks{Equal Contribution}~~\thanks{Google Deepmind}~~\thanks{Correspondence to Khimya Khetarpal <khimya@google.com>.}\\
  \And
  Zhaohan Daniel Guo\footnotemark[1]~~\footnotemark[2]\\
  \And
  Bernardo Avila Pires\footnotemark[2]\\
  \AND
  Yunhao Tang\footnotemark[2]\\
  \And
  Clare Lyle\footnotemark[2]\\
  \And
  Mark Rowland\footnotemark[2]\\
  \And
  Nicolas Heess\footnotemark[2]\\
  \AND
  Diana Borsa\footnotemark[2]\\
  \And
  Arthur Guez\footnotemark[2]\\
  \And
  Will Dabney\footnotemark[2]
}

\begin{document}

\maketitle

\begin{abstract}
  Learning a good representation is a crucial challenge for Reinforcement Learning (RL) agents. Self-predictive learning provides means to jointly learn a latent representation and dynamics model by bootstrapping from future latent representations (BYOL). Recent work has developed theoretical insights into these algorithms by studying a continuous-time ODE model for self-predictive representation learning under the simplifying assumption that the algorithm depends on a fixed policy (BYOL-$\Pi$); this assumption is at odds with practical instantiations of such algorithms, which explicitly condition their predictions on future actions. In this work, we take a step towards bridging the gap between theory and practice by analyzing an action-conditional self-predictive objective (BYOL-AC) using the ODE framework, characterizing its convergence properties and highlighting important distinctions between the limiting solutions of the BYOL-$\Pi$ and BYOL-AC dynamics. We show how the two representations are related by a variance equation. This connection leads to a novel variance-like action-conditional objective (BYOL-VAR) and its corresponding ODE. We unify the study of all three objectives through two complementary lenses; a model-based perspective, where each objective is shown to be equivalent to a low-rank approximation of certain dynamics, and a model-free perspective, which establishes relationships between the objectives and their respective value, Q-value, and advantage function. Our empirical investigations, encompassing both linear function approximation and Deep RL environments, demonstrates that BYOL-AC is better overall in a variety of different settings. %
\end{abstract}

\section{Introduction}
\label{sec:analysis}
Learning a \textit{meaningful} representation and a \textit{useful} model of the world are among the key challenges in reinforcement learning (RL). 
Self-predictive learning has facilitated representation learning often by training auxiliary tasks ~\citep{lee2021predicting}, and making predictions on future observations ~\citep{schrittwieser2020mastering} geared towards control~\citep{jaderberg2016reinforcement, song2019v}. The bootstrap-your-own-latent~[BYOL] framework~\citep{grill2020bootstrap}  together with its RL variant~\citep{guo2020bootstrap} [BYOL-RL] offers a self-predictive paradigm for learning representations by minimizing the prediction error of its own future latent representations. %
Despite empirical advancements~\citep{schwarzer2020data, guo2022byol}, using BYOL for learning transition dynamics $P$ in conjunction with the state representation $\Phi$ remains under-investigated from a theoretical perspective. Enriching our understanding of the self-predictive learning objectives could potentially yield new insights into: a) characterizing which representations are better suited to be used for state-value (V), action-value (Q), or advantage functions; b) identifying the types of representations to which different objectives converge and connections to the corresponding transition dynamics; c) understanding the trade-offs induced by learning objectives in various settings; and d) using these insights to create algorithms that learn meaningful representations. Our work characterizes the ODE dynamics of various BYOL objectives in the context of Markov decision processes (MDPs). %

Previous work~\citep{tang2022understanding} provides initial important theoretical insights by considering a two-timescale, semi-gradient objective and analyzes it from an ODE perspective. A notable component of this objective [BYOL-$\Pi$], considers making a future prediction conditioned on a fixed policy $\pi$. This is in contrast to implementations commonly used in practice, where the future prediction is conditioned on the actions~\citep{guo2022byol}. Recently, \citet{ni2024bridging} provide analysis in the action-conditional POMDP case, but do not fully extend the analysis done by~\citet{tang2022understanding}.

In this work, we focus on the MDP setting, where we want to close the gap between the theoretical analysis of \citet{tang2022understanding} and the practical implementations of BYOL by conditioning on the action. We begin by considering an action-conditional BYOL loss (\cref{eq:loss-byol-ac} [BYOL-AC]) and analyze it in the same theoretical ODE framework. We precisely characterize what representation the ODE converges to (\cref{thm:ode-byol-ac}), showing that the learned representation $\PhiAC$ captures useful spectral information about the per-action transition dynamics $T_{a}$ as opposed to $T^\pi$ corresponding to the non-action-conditional BYOL representation $\Phi$. Furthermore, we show a variance relation between $\PhiAC$ and $\Phi$, where $\Phi$ is related to the square of the first moment of the eigenvalues of $T_a$, and $\PhiAC$ is related to the second moment.
Based on this variance relation, we introduce a novel, variance-like action-conditional objective, BYOL-VAR (\cref{eq:phi-dynamics-vl}). Under the ODE framework, we show the learned representation $\PhiVARAC$ (\cref{thm:ode-byol-var}) is related to the variance of the eigenvalues of $T_a$ (\cref{rem:variance-complete}).

We then unify the study of all three objectives through two complementary lenses. From the model-based viewpoint: we show a certain equivalence of BYOL-$\Pi$, BYOL-AC, BYOL-VAR to learning a low-rank approximation of the dynamics $T^{\pi}$, $T_a$, and $(T_a - T^{\pi})$ respectively (\cref{thm:byol-variants-modelbased-equiv}). From the model-free viewpoint: we show an equivalence to fitting certain 1-step value, Q-value, and advantage functions respectively (\cref{thm:byol-model-free-equiv}). The unified viewpoints 1) give us insights into the suitability of each representation for capturing different aspects of the dynamics; such as using $\Phi$ when we are concerned with $T^{\pi}$, $\PhiAC$ to discern $T_a$, and $\PhiVARAC$ to approximate $(T_a - T^{\pi})$, and 2) establish a bridge between self-predictive objectives and model-based and model-free objectives; if one proposes a new low-rank approximation objective of some other transition dynamics function, then one can use this bridge to derive a corresponding new self-predictive objective.

Empirically, we first examine a linear setting in \cref{sec:mse-stability-experiments} and show how $\Phi$, $\PhiAC$ and $\PhiVARAC$ fit to the true value, Q-value, and advantage functions (\cref{tab:value-mse}). $\Phi$ and $\PhiAC$ turn out to be very similar in fitting value functions, while $\PhiVARAC$ is the undisputed best fit to the advantage function. $\PhiAC$ is also strictly better than $\Phi$ in fitting the advantage. Finally, we compare the three representations in deep RL with Minigrid, and classic control domains, using a policy-gradient, online algorithm V-MPO~\citep{song2019v} and an off-policy algorithm DQN~\citep{mnih2015human}. We report that for both V-MPO and DQN, $\PhiAC$ is overall better performing. $\PhiVARAC$ is, as expected, a poor representation to use directly for RL since it gives up features useful for the value/Q-value function. The key takeaway is that BYOL-AC is overall a better objective resulting in a better representation $\PhiAC$. 

\section{Preliminaries}
\textbf{Reinforcement Learning.} 
Consider an MDP $\langle  \mathcal{X}, \mathcal{A}, T_a, \gamma  \rangle$, where $\mathcal{X}$ is a finite set of states, $\mathcal{A}$ a finite set of actions, $x, y \in \real^{|\mathcal{X}| \times 1}$ assume tabular state representation where each state is a one hot vector, $T_a \in \real^{|\mathcal{X}| \times |\mathcal{X}|}$ is the per-action transition dynamics defined as $(T_a)_{ij} \coloneqq p(y = j \mid x = i, a)$, and $\gamma\in [0,1)$ the discount factor. 
Given $\pi:\mathcal{X}\rightarrow\probspace(\mathcal{A})$ we let $T^\pi:\mathcal{X}\rightarrow\probspace(\mathcal{X})$ be the state transition kernel induced by the policy $\pi$, that is, $(T^\pi)_{ij} \coloneqq \sum_a \pi(a \mid x = i)(T_a)_{ij}$). Given a (deterministic) reward function of the state $R \in \real^{|\mathcal{X}| \times 1}$, the value function is defined as $V^{\pi} \coloneqq (I - \gamma T^\pi)^{-1} R$. The Q-value function is defined as $Q_a^{\pi} \coloneqq R + \gamma T_a V^{\pi}$.

\textbf{Representation Learning.}
It can be shown under idealized conditions that value function estimation methods such as TD-learning capture the top-$k$ subspace of the eigenbasis of $T^\pi$~\citep{lyle2021effect}. Moreover, the top-$k$ eigenvectors of $T^\pi$ are exactly the same as those of $(I-\gamma T^\pi)^{-1}$ ~\citep{chandak2023representations}. Notably, when the value function is linear in the features, the basis vectors are useful features of the state. Additional related work is provided in~\cref{sec:relatedwork}.

\textbf{Ordinary differential equations (ODEs)} arise often in the context of describing change in the representations over time as they are being learned. \citet{lyle2021effect} consider the dynamics for single-step TD learning and assume that the weight is kept fixed, which greatly simplifies the dynamics. \citet{tang2022understanding} consider the ODEs systems with certain dynamics that make them not the same as traditional optimization problems. To analyze such an ODE system, one constructs a Lyapunov function, which can be considered a surrogate loss function that the ODE monotonically minimizes. Lyapunov functions are also useful to show necessary and sufficient conditions for stability and convergence of the said ODE ~\citep{teschl2012ordinary,MAWHIN2005664}.

\subsection{BYOL ODE with Fixed Policy: BYOL-$\Pi$}
\label{sec:ode-byol}
We start with the setting considered by \citet{tang2022understanding}. We observe triples $(x, a, y)$ where $x \sim d_X$ is a one-hot state, $a \sim \pi(\cdot \mid x)$ is an action sampled according to a fixed policy $\pi$, and $y \sim p(\cdot \mid x, a)$ is the on-hot state observed after taking action $a$ at state $x$. 
Letting $D_{X} \coloneqq \E [xx^T] = \diag (d_X)$, we can write $\E[xy^T] = D_{X} T^{\pi}$. 
The goal is to learn a representation matrix $\Phi \in \mathbb{R}^{|\mathcal{X}| \times k}$ that embeds each state $x$ as a $k$-dimensional real vector denoted by $\Phi^T x$. To model the transitions in latent space, $\Phi^T x \xrightarrow{} \Phi^T y$, we consider a latent linear map $P \in \mathbb{R}^{k \times k}$. 
To learn $\Phi$, we use a self-predictive objective that minimizes the loss in the latent space,
\begin{align}
    \min_{\Phi, P} \; \text{BYOL-}\Pi(\Phi, P) := \E_{x \sim d_X, y \sim T^{\pi}( \cdot \mid x ) } \left[ \big| \big| P^T \Phi^T x - \sg(\Phi^T y)  \big| \big|^{2}_{2} \right] \label{eq:byol-non-ac}
\end{align}
where $\sg$ is a stop-gradient operator on the prediction target to help in avoiding degenerate solutions. 
The appeal of this objective is that everything is defined in the latent space, which means this can be easily extended and implemented in practice with $P$ and $\Phi$ replaced by neural networks. However, this objective still has the trivial solution of $\Phi = 0$. To avoid this, \citet{tang2022understanding} formulate a two-timescale optimization process wherein we first solve the inner minimization w.r.t. (with respect to) $P$ before taking a (semi-)gradient step (denoted as $\dot{\Phi}$) w.r.t. $\Phi$.
\begin{align}
\label{eq:ode-nonactionconditoned}
    \begin{aligned}
     P^{*} &\in \textstyle \argmin_{P}  \; \text{BYOL-}\Pi(\Phi, P) , \ \  \dot{\Phi} = - \nabla_{\Phi} \; \text{BYOL-}\Pi(\Phi, P) \big\vert_{P = P^{*}}
    \end{aligned}
\end{align}
This is an ODE system for $\Phi$ with dynamics (gradient) $\dot{\Phi}$. \citet{tang2022understanding} makes the following simplifying assumptions to analyze it:
\begin{restatable}[Orthogonal Initialization]{reassumption}{restateassortho}
\label{ass:orthogonal-init}
$\Phi$ is initialized to be orthogonal i.e. $\Phi^T \Phi = I$.
\end{restatable}
\begin{restatable}[Uniform State Distribution]{reassumption}{restateassuniformstate}
\label{ass:uniform-state}
The state distribution $d_X$ is uniform.
\end{restatable}
\begin{restatable}[Symmetric Dynamics]{reassumption}{restateasssymmetrict}
\label{ass:symmetric-t}
$T^{\pi}$ is symmetric i.e. $T^{\pi} = (T^{\pi})^T$.
\end{restatable}
While these assumptions are quite strong and impractical, we believe the resulting theoretical insights are a useful perspective in understanding and characterizing the learned representation $\Phi$ in practice.
\begin{restatable}[Non-collapse, \citealp{tang2022understanding}]{relemma}{restatelemnoncollapse}
\label{lem:non-collapse-byol}
Under \cref{ass:orthogonal-init}, we have that $\Phi^T \dot{\Phi} = 0$, which means that $\Phi^T \Phi = I$ is preserved for all $\Phi$ throughout the ODE process.
\end{restatable}
Intuitively, \cref{lem:non-collapse-byol} suggests that because of how we set up the ODE with the semi-gradient and two-timescale optimization, an orthogonal initialization means we can avoid all trivial solutions.

\begin{restatable}[BYOL Trace Objective, \citealp{tang2022understanding}]{relemma}{restatelemtrace}
\label{lem:trace-byol}
Under \cref{ass:orthogonal-init,ass:uniform-state,ass:symmetric-t}, a Lyapunov function for the ODE is the negative of the following trace objective
\begin{align}
\label{eq:trace-byol}
    f_{\text{BYOL-}\Pi}(\Phi) &\coloneqq  \mtrace \left(\Phi^T T^{\pi} \Phi  \Phi^T T^{\pi} \Phi \right) \, .
\end{align}
This means the ODE converges to some critical point.
\end{restatable}
By construction, the critical points of the ODE are also critical points of the latent loss, so \cref{lem:trace-byol} establishes that the ODE converges to such a non-collapsed critical point.

\begin{restatable}[BYOL-$\Pi$ ODE, \citealp{tang2022understanding}]{retheorem}{restatethmode}
\label{thm:ode-byol}
Under~\cref{ass:orthogonal-init,ass:uniform-state,ass:symmetric-t}, let $\Phi^*$ be any maximizer of the trace objective $f_{\text{BYOL-}\Pi}(\Phi)$:
\begin{align}
    \textstyle \Phi^{*} \subseteq \argmax_{\Phi} \; f_{\text{BYOL-}\Pi}(\Phi) = \argmax_{\Phi} \; \mtrace \left( \Phi^T T^{\pi} \Phi \Phi^T T^{\pi} \Phi \right) \, .
\end{align}
Then $\Phi^*$ is a critical point of the ODE. Furthermore, the columns of $\Phi^*$ span the same subspace as the top-$k$ eigenvectors of $(T^{\pi})^2$.
\end{restatable}
This trace objective $f_{\text{BYOL-}\Pi}(\Phi)$ is essentially a surrogate loss function that the ODE is monotonically maximizing, that also has the same critical points by construction. Thus to understand the ODE, we simply analyze the maximizer of the trace objective. In this case, the maximizer $\Phi^{*}$ learns important features (eigenvectors) of the transition dynamics $T^{\pi}$. We demonstrate this in Appendix~\cref{fig:tracesplot}, for both symmetric and non-symmetric MDPs.

\section{Understanding Action-Conditional BYOL}
\label{sec:byol-ac}

\subsection{The Action-Conditional BYOL Objective}
\label{sec:ode-byol-ac}

We start by identifying a key distinction between the theoretical analysis and practical implementations of BYOL in the literature. Empirical investigation of BYOL variants~\citep{grill2020bootstrap, guo2022byol, tang2022understanding} typically use an action-conditioned objective. However, the analytical framework for self-predictive learning~\citep{tang2022understanding} considers a policy dependent objective function marginalizing over actions. We refer to the non-action conditional objective as BYOL-$\Pi$ henceforth. The practical success of the action-conditional BYOL referred to as BYOL-AC henceforth, serves as our primary motivation to develop a better understanding of it analytically.

We now formulate the BYOL-AC self-predictive ODE and analyze it in a similar manner to BYOL-$\Pi$ ODE (\cref{sec:ode-byol}). The setup is the same as before except we now explicitly use action-conditional predictors, that is a predictor $P_a$ per action instead of a single predictor $P$ in BYOL-$\Pi$. 
The goal is then to minimize the following reconstruction loss in the latent space:
\begin{align}
\label{eq:loss-byol-ac}
    \min_{\Phi, \{ \forall P_a \}} \text{BYOL-AC}(\Phi, P_{a_1}, P_{a_2}, \dots) := \mathbb{E}_{x \sim d_X, a \sim \pi(\cdot | x), y \sim T_a(\cdot | x)} \left[ \left\Vert P_a^T \Phi^T x - \sg(\Phi^T y) \right\Vert^2 \right]
\end{align}
This is a natural extension of the BYOL-$\Pi$ ODE in \cref{eq:byol-non-ac} to explicitly condition the predictions on actions. The ODE system we consider is then a similar two-timescale optimization as before, where we first solve for the optimal $P_a$, followed by a gradient step for $\Phi$:
\begin{align}
\label{eq:ode-byol-ac}
    \begin{aligned}
    \forall a: \; P_a^{*} &\in \textstyle \argmin_{P_a}  \text{BYOL-AC}(\Phi, P_{a}), \quad \dot{\Phi} = - \nabla_{\Phi} \; \text{BYOL-AC}(\Phi, P_{a}) \big\vert_{P_a = Pa^{*}}
    \end{aligned}
\end{align}

Next, to analyze this ODE, we make a few new assumptions in addition to~\cref{ass:orthogonal-init,ass:uniform-state,ass:symmetric-t}. For additional intuition on theory, we defer the reader to~\Cref{sec:intuition_theory}. All proofs are in~\Cref{sec:app-proofs}.
\begin{restatable}[Uniform Policy]{reassumption}{restateassuniformpolicy}
\label{ass:uniform-policy}
The (data-collection) policy $\pi$ is uniform across all actions.
\end{restatable}
We also make an analogue of~\cref{ass:symmetric-t} for the action-conditional setting:
\begin{restatable}[Symmetric Per-action Dynamics]{reassumption}{restateasssymmetricta}
\label{ass:symmetric-ta}
$T_a$ is symmetric for all actions i.e. $T_a = (T_a)^T$.
\end{restatable}

We establish analogues of \cref{lem:non-collapse-byol,lem:trace-byol,thm:ode-byol} for 
BYOL-AC, that is, the non-collapse property of BYOL-AC, a trace objective that describes a Lyapunov function for the ODE, and a main theorem that helps us understand what kind of representation BYOL-AC learns.
\begin{restatable}[Non-collapse BYOL-AC]{relemma}{restatelemnoncollapseac}
\label{lem:non-collapse-byol-ac}
Under \cref{ass:orthogonal-init}, we have that $\Phi^T \dot{\Phi} = 0$, which means that $\Phi^T \Phi = I$ is preserved for all $\Phi$ throughout the BYOL-AC ODE process. 
\end{restatable}
\begin{restatable}[BYOL-AC Trace Objective]{relemma}{restatelemtraceac}
\label{lem:trace-byol-ac}
Under \cref{ass:orthogonal-init,ass:uniform-state,ass:symmetric-t,ass:uniform-policy,ass:symmetric-ta}, a Lyapunov function for the BYOL-AC ODE is the negative of the following trace objective
\begin{align}
\label{eq:trace-byol-ac}
    f_{\text{BYOL-AC}}(\Phi) &\coloneqq \textstyle |A|^{-1}\sum_a  \mtrace \left(\Phi^T T_a \Phi \Phi^T T_a \Phi \right)
\end{align}
This means the ODE converges to some critical point.
\end{restatable}
Before presenting our main theorem, we require the following assumption.
\begin{restatable}[Common Eigenvectors]{reassumption}{restateasscommonta}
\label{ass:common-ta}
For all actions $a$, we have the eigen decomposition $T_a = Q D_a Q^T$, i.e. all $T_a$ share the same eigenvectors.
\end{restatable}
\begin{restatable}[BYOL-AC ODE]{retheorem}{restatethmodeac}
\label{thm:ode-byol-ac}
Under \cref{ass:orthogonal-init,ass:uniform-state,ass:symmetric-t,ass:uniform-policy,ass:symmetric-ta,ass:common-ta}, let $\PhiAC^*$ be any maximizer of the trace objective $f_{\text{BYOL-AC}}(\Phi)$:
\begin{align}
    \textstyle \PhiAC^{*} \subseteq \argmax_{\Phi} \; f_{\text{BYOL-AC}}(\Phi) = \argmax_{\Phi} \;  |A|^{-1}\sum_a \mtrace \left(\Phi^T T_a \Phi \Phi^T T_a \Phi \right)
\end{align}
Then $\PhiAC^*$ is a critical point of the ODE. Furthermore, the columns of $\PhiAC^*$ span the same subspace as the top-$k$ eigenvectors of $\left( |A|^{-1} \sum_a T_a^2 \right)$.
\end{restatable}
We have an analogous result to~\Cref{thm:ode-byol-ac}, but instead of $(T^{\pi})^2$, the columns of $\PhiAC^*$ span the same subspace as $k$ eigenvectors of $Q$ corresponding to the top-$k$ eigenvalues of $\left( |A|^{-1} \sum_a T_a^2 \right)$.

\subsection{Comparing the Representation Learned by BYOL-$\Pi$ and BYOL-AC}
\label{sec:action-conditional-odeandtrace}

In this section, we will characterize the difference between representations $\Phi^{*}$ and $\PhiAC^{*}$ achieved by BYOL and BYOL-AC. Under~\cref{ass:orthogonal-init,ass:uniform-state,ass:symmetric-t,ass:uniform-policy,ass:symmetric-ta,ass:common-ta}, we have the eigendecomposition $T_a = Q D_a Q^T$ for all actions $a$. Because we have a uniform policy, we also know that $T^{\pi} = |A|^{-1} \sum_a T_a = Q \left( |A|^{-1} \sum_a D_a \right) Q^T$. Since $T_a$ and $T^{\pi}$ have the same eigenvectors $Q$, we know that both $\Phi^{*}$ and $\PhiAC^{*}$ will correspond to a subset of $k$ eigenvectors from $Q$. The only difference is that they use different criteria to pick eigenvectors.

From~\cref{thm:ode-byol}, we know that $\Phi^{*}$ picks according to the eigenvalues of $(T^{\pi})^2$ i.e. from $\left( |A|^{-1} \sum_a D_a \right)^2$ since $T^{\pi} = |A|^{-1} \sum_a T_a$. From~\cref{thm:ode-byol-ac}, we know that $\PhiAC^{*}$ picks according to $\left( |A|^{-1} \sum_a D_a^2 \right)$. One is the square of the mean, and the other is the mean of the squares. We can relate the two quantities with the following.
\begin{restatable}[Variance Relation]{reremark}{restateremvar}
\label{rem:variance}
\begin{align*}
    \underbrace{\mathbb{E}_{a \sim \mathrm{Unif}} \left[ D_a^2 \right]}_{\text{BYOL-AC}} &= \underbrace{\left( \mathbb{E}_{a \sim \mathrm{Unif}} \left[ D_a \right] \right)^2}_{\text{BYOL}-\Pi} + \mathrm{Var}_{a \sim \mathrm{Unif}}(D_a)
\end{align*}
\end{restatable}
Thus BYOL-AC picks eigenvectors that are not only good according to BYOL, but also have large variance of eigenvalues across actions. Intuitively, this means BYOL-AC pays attention to features that can distinguish between actions.

\begin{tcolorbox}
[colback=lightgray!1!white,colframe=lightgray!75!black]
\textbf{Key Insight.} BYOL-AC captures the spectral information about per-action transition matrices $T_a$,  BYOL-$\Pi$ captures the spectral information about the policy induced transition matrix $T^{\pi}$. 
\end{tcolorbox}

\section{Variance-Like Action-Conditional BYOL}
\label{sec:byol-var}

Comparing BYOL and BYOL-AC, we note that we can relate their trace objective maximizers using a variance equation (\cref{rem:variance}), with BYOL being the square of the first moment, and BYOL-AC being the second moment. This poses a natural question: Is there an objective corresponding to the variance term i.e. the difference between the second moment and the square of the first moment? We answer this question in the affirmative by proposing a new variance-like BYOL objective:
\begin{equation}
    \label{eq:phi-dynamics-vl}
    \textstyle \min_{\Phi} \text{BYOL-VAR}(\Phi, P, P_{a_1}, P_{a_2}, \dots) \coloneqq \mathbb{E}\left[\| P_a^\top \Phi^\top x - \sg(\Phi^\top y) ) \|^2 - \| P^\top \Phi^\top x - \sg(\Phi^\top y) ) \|^2\right]
\end{equation}
But now with the predictors as before, solving
\[
    \textstyle \min_P \mathbb{E}\left[\| P^\top \Phi^\top x - \Phi^\top y \|^2\right], \qquad \text{and} 
    \textstyle \qquad  \forall a: \min_{P_a} \mathbb{E}\left[\| P_a^\top \Phi^\top x - \Phi^\top y \|^2\right].
\]
The BYOL-VAR objective is a difference of the BYOL-AC and BYOL-$\Pi$ objectives. Analogous to our previous results, we derive the corresponding ODE dynamics, with statements about non-collapse, a Lyapunov function, and a result about what the representation captures.
\begin{gather}
\label{eq:ode-byol-var-ac}
P^{*} \in \textstyle \argmin_{P}  \E \left[ \big| \big| P^T \Phi^T x - \sg(\Phi^T y)  \big| \big|^{2} \right], \quad \forall a: P_a^{*} \in \argmin_{P_a}  \E \left[ \big| \big| P_a^T \Phi^T x - \sg(\Phi^T y)  \big| \big|^{2} \right] \nonumber \\
\dot{\Phi} = - \nabla_{\Phi} \text{BYOL-VAR}(\Phi, P, P_{a_1}, P_{a_2}, \dots) \big\vert_{P=P^{*}, P_a = Pa^{*}}
\end{gather}
\begin{restatable}[Non-collapse BYOL-VAR]{relemma}{restatelemnoncollapsevar}
\label{lem:non-collapse-byol-var}
Under \cref{ass:orthogonal-init,ass:uniform-policy}, we have that $\Phi^T \dot{\Phi} = 0$, which means that $\Phi^T \Phi = I$ is preserved for all $\Phi$ throughout the BYOL-VAR ODE process.
\end{restatable}
\begin{restatable}[BYOL-VAR Trace Objective]{relemma}{restatelemtracevar}
\label{lem:trace-byol-var}
Under \cref{ass:orthogonal-init,ass:uniform-state,ass:symmetric-t,ass:uniform-policy,ass:symmetric-ta}, a Lyapunov function for the BYOL-VAR ODE is the negative of the following trace objective
\begin{align}
\label{eq:trace-byol-var}
    f_{\text{BYOL-VAR}}(\Phi) &\coloneqq f_{\text{BYOL-AC}}(\Phi) - f_{\text{BYOL-}\Pi}(\Phi) \nonumber \\
    &= |A|^{-1}\sum_a  \mtrace \left(\Phi^T T_a \Phi \Phi^T T_a \Phi\right)  - \mtrace \left(\Phi^T T^{\pi} \Phi \Phi^T T^{\pi} \Phi  \right)
\end{align}
This means the ODE converges to some critical point.
\end{restatable}

\begin{restatable}[BYOL-VAR ODE]{retheorem}{restatethmodevar}
\label{thm:ode-byol-var}
Under \cref{ass:orthogonal-init,ass:uniform-state,ass:symmetric-t,ass:uniform-policy,ass:symmetric-ta,ass:common-ta}, let $\Phi_{\text{VAR}}^*$ be any maximizer of the trace objective $f_{\text{BYOL-VAR}}(\Phi)$:
\begin{align}
    \textstyle \Phi_{\text{VAR}}^* \subseteq \argmax_{\Phi} \; |A|^{-1}\sum_a  \mtrace \left(\Phi^T T_a \Phi \Phi^T T_a \Phi\right)  - \mtrace \left(\Phi^T T^{\pi} \Phi \Phi^T T^{\pi} \Phi  \right)
\end{align}
Then $\Phi_{\text{VAR}}^*$ is a critical point of the ODE. Furthermore, the columns of $\Phi_{\text{VAR}}^*$ span the same subspace as the top-$k$ eigenvectors of $\left( |A|^{-1} \sum_a T_a^2 - (T^{\pi})^2\right)$.
\end{restatable}

Like before, all $\Phi^*$, $\PhiAC^*$, and $\PhiVARAC^*$ correspond to subsets of eigenvectors from the same eigenspace of $Q$. Now we can complete this story of the variance relationship in how they pick eigenvectors.
\begin{tcolorbox}
[colback=lightgray!1!white,colframe=lightgray!75!black]
\begin{restatable}[Complete Variance Relation]{reremark}{restatermvarcomplete}
\label{rem:variance-complete}
\begin{align*}
    \underbrace{\mathbb{E}_{a \sim \mathrm{Unif}} \left[ D_a^2 \right]}_{\text{BYOL-AC}} &= \underbrace{\left( \mathbb{E}_{a \sim \mathrm{Unif}} \left[ D_a \right] \right)^2}_{\text{BYOL}-\Pi} + \underbrace{\mathrm{Var}_{a \sim \mathrm{Unif}}(D_a)}_{\text{BYOL-VAR}}
\end{align*}
\end{restatable}
\end{tcolorbox}

Intuitively, BYOL-VAR tries to learn a representation $\PhiVARAC^{*}$ that only captures features for distinguishing between actions.  In practice, our assumptions are unlikely to be satisfied. However, the intuition behind the variance relation (\cref{rem:variance-complete}) still gives us a valuable insight: $\Phi^*$ is concerned with meaningful features for $T^{\pi}$; $\PhiAC^{*}$ tries to capture meaningful features of $T_a$; $\PhiVARAC$ tries to only capture features that can distinguish across $T_a$.

\Cref{fig:picking_eigenvectors} shows an illustrative MDP with two actions demonstrating which eigenvectors each objective converges to. In the next section, we will present two unifying perspectives for comparing between the three objectives of BYOL-$\Pi$, BYOL-AC, and BYOL-VAR. \begin{figure}[h]
    \centering
    \includegraphics[width=0.28\columnwidth]{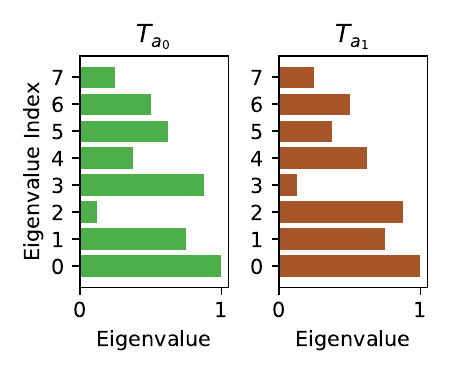}
    \includegraphics[width=0.28\columnwidth]{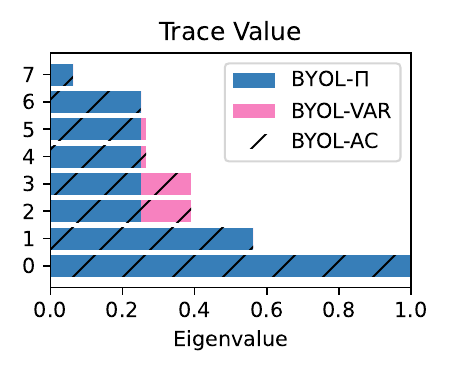}
    \includegraphics[width=0.28\columnwidth]{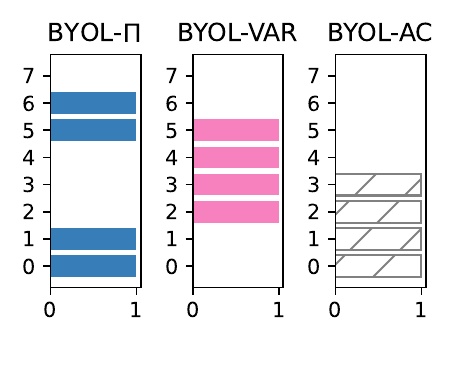}
    \caption{\textbf{On the representations across BYOL-$\Pi$, BYOL-AC, and BYOL-VAR.} We consider a simple MDP with two actions and corresponding transition functions $T_{a_0}, T_{a_1}$, with the eigenvalues of each action depicted in two leftmost plots. The middle plot shows a stacked bar plot of the trace objective values corresponding to each objective. The three rightmost plot shows each objective picking its top-$k$ ($k=4$) eigenvectors.}
    \label{fig:picking_eigenvectors}
\end{figure}

\section{Two Unifying Perspectives: Model-Based and Model-Free}
\label{sec:unifying}
The variance relationship between the three objectives BYOL-$\Pi$, BYOL-AC, and BYOL-VAR (\cref{rem:variance-complete}) provides a natural connection between the learned representations. However, it is somewhat abstract as it concerns itself with eigenvalues and eigenvectors under strict assumptions. In this section, we will unify the study of all three three objectives through two complimentary lenses, namely, a model-based perspective and a model-free perspective.

\subsection{Fitting Dynamics - A Model-Based View}
\label{sec:model-based-unified-view}
From the model-base perspective, we can derive an equivalence between each of the trace objectives akin to finding a low-rank approximation of certain transition dynamics.
\begin{restatable}[Unifying Model-Based View]{retheorem}{restatethmmodelbased}
\label{thm:byol-variants-modelbased-equiv}
Under~\Cref{ass:orthogonal-init,ass:uniform-state,ass:symmetric-t,ass:uniform-policy,ass:symmetric-ta,ass:common-ta}, the negative trace objectives of BYOL-$\Pi$, BYOL-AC, and BYOL-VAR are equivalent (up to a constant $\mathrm{C}$) to the following objectives ($\Vert \cdot \Vert_F$ is the Frobenius matrix norm):
\begin{align}
    - f_{\text{BYOL-}\Pi}(\Phi) &= \textstyle \min_P \; \Vert T^{\pi} - \Phi P \Phi^T \Vert_F + \mathrm{C} \label{eq:model-based-byol} \\
    - f_{\text{BYOL-AC}}(\Phi) &= \textstyle |A|^{-1} \sum_a \min_{P_a} \; \Vert T_a - \Phi P_a \Phi^T \Vert_F + \mathrm{C}  \label{eq:model-based-byol-ac} \\
    - f_{\text{BYOL-VAR}}(\Phi) &= \textstyle |A|^{-1} \sum_a \min_{P_{\Delta a}} \; \Vert (T_a - T^{\pi}) - \Phi P_{\Delta a} \Phi^T \Vert_F + \mathrm{C} \label{eq:model-based-byol-var}
\end{align}
Therefore, maximizing the trace (over orthogonal $\Phi$) results in BYOL-$\Pi$, BYOL-AC, and BYOL-VAR trying to fit a low-rank approximation of the dynamics matrix $T^{\pi}$, per-action transition matrix $T_a$, and the residual dynamics $(T_a - T^{\pi})$ respectively.
\end{restatable}
Thus we now have a direct relationship between BYOL-$\Pi$, BYOL-AC, and BYOL-VAR and $T^{\pi}$, $T_a$, and $(T_a - T^{\pi})$ respectively. The three ODEs are trying to learn good latent dynamics models. In this paper we focus on the representation $\Phi$, but future work could further investigate the learned dynamics $P$ and $P_a$, e.g. from a planning perspective.

\subsection{Fitting Value Functions - A Model-Free View}
\label{sec:mse-unifiedview}

Complimentary to the model-based view, we can rewrite the the maximizer to the trace objectives through a model-free lens. To do this, we assume an isotropic Gaussian reward function $R$ i.e. $E[RR^T] = |\mathcal{X}|^{-1} I$. One could try to fit the learned representation $\Phi$ to the value function $V$ as such
$\E_R \left[ \min_{\theta} \; \Vert V - \Phi \theta \Vert^2 \right] = \E_R \left[ \min_{\theta} \; \Vert (I - T^{\pi})^{-1}R - \Phi \theta \Vert^2 \right]$. In a similar vein, we can re-express the maximizers to the trace objectives as various 1-step value-like functions over these reward functions $R$.

\begin{restatable}[Unifying Model-Free View]{retheorem}{restatethmmodelfree}
\label{thm:byol-model-free-equiv}
Under~\Cref{ass:orthogonal-init,ass:uniform-state,ass:symmetric-t,ass:uniform-policy,ass:symmetric-ta,ass:common-ta}, the negative trace objectives of BYOL-$\Pi$, BYOL-AC, and BYOL-VAR are equivalent (up to a constant $\mathrm{C}$) to the following objectives:
\begin{align}
\label{eq:valuemse-objective}
    - f_{\text{BYOL-}\Pi}(\Phi) &= \textstyle |\mathcal{X}| \mathbb{E} \left[ \min_{\theta, \omega} \; \left( \Vert T^\pi R - \Phi \theta \Vert^2 + \Vert T^\pi \Phi \Phi^T R - \Phi \omega \Vert^2 \right) \right]  + \mathrm{C}  \\
\label{eq:qvaluemse-objective}
    - f_{\text{BYOL-AC}}(\Phi) &= \textstyle |\mathcal{X}| \mathbb{E} \left[ |A|^{-1} \sum_a \min_{\theta_a, \omega_a} \; \left( \Vert T_a R - \Phi \theta_a \Vert^2 + \Vert T_a \Phi \Phi^T R - \Phi \omega_a \Vert^2  \right) \right]  + \mathrm{C}  \\
\label{eq:advantagemse-objective}
    - f_{\text{BYOL-VAR}}(\Phi) &= \textstyle |\mathcal{X}| \mathbb{E} \big[ |A|^{-1} \sum_a \min_{\theta_a, \omega_a} \; \big( \Vert (T_a R - T^{\pi} R) - \Phi \theta \Vert^2 \nonumber \\
     &\qquad\qquad\qquad\qquad+  \Vert (T_a \Phi \Phi^T R - T^{\pi} \Phi \Phi^T R) - \Phi \omega \Vert^2 \big) \big]  + \mathrm{C} 
\end{align}
Therefore, maximizing the trace (over orthogonal $\Phi$) results in BYOL-$\Pi$, BYOL-AC, and BYOL-VAR trying to fit a certain 1-step value (V), Q-value, and Advantage function respectively.
\end{restatable}

Note that there are two parts to the value-like function objective. The one on the left is the $1$-step value function where the agent takes one action and transitions to a final state with reward. The one on the right is similar, but instead of the actual reward, the agent obtains a projected reward $\Phi\Phi^T R$, which is the optimal approximation when trying to use the representation $\Phi$ to fit the reward function. For BYOL-AC, it is similar except we try to fit the 1-step Q-value-like function. Finally, BYOL-VAR tries to fit the 1-step advantage-like function.

\begin{tcolorbox}
[colback=lightgray!1!white,colframe=lightgray!75!black]
\textbf{Key Insight.} The various BYOL objectives can be unified through two complimentary lens: a model-based view, suggesting that BYOL-$\Pi$, BYOL-AC, and BYOL-VAR are best at capturing information about $T^\pi$, $T_a$, and $(T_a-T^\pi)$ respectively, and a model-free view, suggesting that each method minimizes its corresponding negative trace objectives, and are trying to fit a certain 1-step value (V), Q-value, and Advantage function respectively.   
\end{tcolorbox}

Notably, while the model-based and the model-free perspectives offer complimentary views, both highlight the three ODE objectives yield useful representations for various important RL quantities. However, this theory is not guaranteed to translate to empirical RL performance, because of many assumptions. Therefore, in the remainder of the paper, we will examine the learned representations in both linear function approximation and deep RL settings.

\section{Experiments} 
Having studied the learning dynamics of the three objectives theoretically, we next pose empirical questions about the BYOL-objectives affecting the RL performance in both linear (Sec.~\ref{sec:mse-stability-experiments}) and non-linear (Sec.~\ref{sec:deeprl}) function approximation settings. 

\subsection{Linear Function Approximation}
\label{sec:mse-stability-experiments}
First, we corroborate~\Cref{thm:byol-model-free-equiv} and~\Cref{thm:byol-variants-modelbased-equiv} empirically in a linear function approximation setting.
We consider randomly generated MDPs with $10$ states, $4$ actions and symmetric per-action dynamics $T_a$.
We learn a compressed representation with dimension 4 for each of BYOL-$\Pi$, BYOL-AC, and BYOL-VAR. Results in this section are averaged over $100$ runs with $95$\% standard error. \Cref{tab:unifyingmodelfree-linearexp} shows the values of the three negative trace objectives (rows) versus the representation learned by the three methods (columns). As predicted by the theory, we see that the smallest negative trace objective is attained by the corresponding ODE. $\Phi$ minimizes $-f_{\text{BYOL-}\Pi}$ (\cref{eq:model-based-byol,eq:valuemse-objective}) i.e. \textbf{Pr($\Phi$ is best)} is $99\%$,  $\PhiAC$ minimizes $-f_{\text{BYOL-AC}}$ (\cref{eq:model-based-byol-ac,eq:qvaluemse-objective}) i.e. \textbf{Pr($\PhiAC$ is the best)} $99\%$, whereas $\PhiVARAC$ minimizes  $-f_{\text{BYOL-VAR}}$ (\cref{eq:model-based-byol-var,eq:advantagemse-objective}) i.e. \textbf{Pr($\PhiVARAC$ is the best)} is $100\%$. 

\begin{table}[h!]
\caption{\textbf{Illustrating~\Cref{thm:byol-variants-modelbased-equiv} and \Cref{thm:byol-model-free-equiv}} empirically demonstrates that each method minimizes its corresponding negative trace objectives, which means BYOL-$\Pi$, BYOL-AC, and BYOL-VAR are best at capturing information about $T^\pi$, $T_a$, and $(T_a-T^\pi)$ respectively, and are trying to fit a certain 1-step value (V), Q-value, and Advantage function respectively.}
\label{tab:unifyingmodelfree-linearexp}
\centering
\begin{adjustbox}{width=\columnwidth,center}
\begin{tabular}{c|cc|cc|cc}
\hline
\cellcolor{aliceblue} \textbf{Method}   & \multicolumn{2}{c|}{\cellcolor{aliceblue} \textbf{BYOL-$\Pi$} [$\Phi$]} & \multicolumn{2}{c|}{\cellcolor{aliceblue}  \textbf{BYOL-AC} [$\PhiAC$]}  & \multicolumn{2}{c|}{\cellcolor{aliceblue} \textbf{BYOL-VAR } [$\PhiVARAC$]}  \\
\hline
\textbf{Objective}    &  &  \textbf{Pr($\Phi$ is best)} &    &  \textbf{Pr($\PhiAC$ is best)} &   & \textbf{Pr($\PhiVARAC$ is best)} \\ \hline
$-f_{\text{BYOL-}\Pi}$   & $-1.22 \pm 0.00$  & \cellcolor{lavender} \textbf{99}\%  & $-1.10\pm 0.01$ & 1\%  & $-0.04\pm 0.00$ & 0.\%       \\
$-f_{\text{BYOL-AC}}$ & $-1.31\pm 0.01$  & 1\%   & $-1.44\pm 0.00$ & \cellcolor{lavender} \textbf{99}\%  & $-0.55\pm 0.01$ & 0.\% \\
$-f_{\text{BYOL-VAR}}$ & $ -0.09 \pm 0.00$  & 0\%   & $-0.33 \pm 0.00$ & 0\%  & $-0.50 \pm 0.01$ & \cellcolor{lavender}\textbf{100}\% \\
\hline
\end{tabular}
\end{adjustbox}
\end{table}

Next, we consider the same three methods and fit the traditional V-MSE ($\E_R \left[ \min_{\theta} \; \Vert V - \Phi\theta \Vert^2 \right]$), Q-MSE and Advantage-MSE. \Cref{tab:value-mse} illustrates that both BYOL-$\Pi$ and BYOL-AC perform competitively in fitting the state-value reporting a V-MSE of $6.32$, and $6.48$ respectively, while fitting an action-value suffering a Q-MSE of $8.31$, and $8.01$ respectively. BYOL-VAR instead learns $\PhiVARAC$ which turns out be optimal for fitting the true Advantage MSE observed to be $0.43$ and \textbf{Pr($\PhiVARAC$ is best)} to be $100\%$.

\begin{table}[h!]
\caption{\textbf{Fitting various value functions to learned representations} for $\Phi$, $\PhiAC$, and $\PhiVARAC$. We report both the MSE and the probability of a representation being best.}
\label{tab:value-mse}
\centering
\begin{adjustbox}{width=\columnwidth,center}
\begin{tabular}{c|cc|cc|cc}
\hline
\cellcolor{aliceblue} \textbf{Method}   & \multicolumn{2}{c|}{\cellcolor{aliceblue} \textbf{BYOL-$\Pi$} [$\Phi$]} & \multicolumn{2}{c|}{\cellcolor{aliceblue}  \textbf{BYOL-AC} [$\PhiAC$]}  & \multicolumn{2}{c|}{\cellcolor{aliceblue} \textbf{BYOL-VAR } [$\PhiVARAC$]} \\
\hline
\textbf{Objective}    &  &  \textbf{Pr($\Phi$ is best)} &    &  \textbf{Pr($\PhiAC$ is best)} &   & \textbf{Pr($\PhiVARAC$ is best)} \\ \hline
V-MSE   & $6.32 \pm 0.06$ & 59\%  & $6.48 \pm 0.05$ & 41\%  & $10005.53 \pm 0.05$ & 0\%   \\
Q-MSE            & $8.31 \pm 0.35$  & 52\% & $8.01 \pm 0.30$ & 48\% & $10005.97 \pm 0.05$ & 0\%   \\
Advantage-MSE    & $0.76 \pm 0.01$  & 0\%  & $0.61 \pm 0.01$ & 0\%  & $0.43 \pm 0.01$ & \cellcolor{lavender} \textbf{100}\%  \\
\hline
\end{tabular}
\end{adjustbox}
\end{table}

Besides, we investigated the robustness of each representation to perturbations in the initial policy used to learn the representation in \cref{appsec:robustnessofphi-to-changes}. We report that $\PhiAC$ learned by BYOL-AC objective is much more robust to changes in the policy compared to BYOL-$\Phi$ and BYOL-VAR.

\subsection{Deep Reinforcement Learning}
\label{sec:deeprl}
In this section, we investigate the effects of the three objectives with V-MPO and DQN. Results in this section are averaged over 10 independent seeds with 95\% standard error in the error bands. We defer details on domains and hyper parameter tuning to ~\cref{appsec:minigrid} and \ref{appsec:openaigym}.

\textbf{V-MPO.} 
\label{sec:experiments-on-policy}
\begin{figure}[h!]
    \centering
    \includegraphics[width=1.0\linewidth]{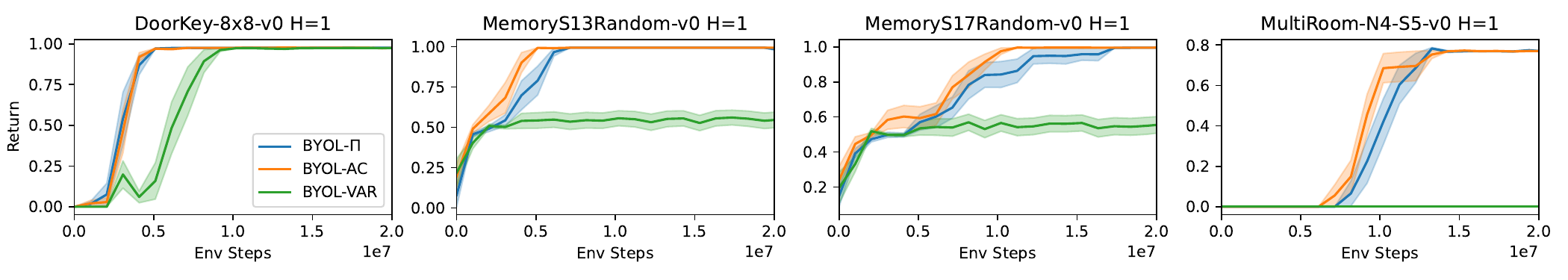}
    \caption{Comparing BYOL-$\Pi$, BYOL-AC, and BYOL-VAR augmented with a V-MPO agent in Minigrid. $\PhiAC$ is overall better than $\Phi$, whereas $\PhiVARAC$ is a weak baseline and struggles.}
    \label{fig:minigrid-deeprl}
\end{figure} We modify the implementation of a V-MPO~\citep{song2019v} agent by augmenting it with the auxiliary loss corresponding to BYOL-$\Pi$, BYOL-AC, and BYOL-VAR. We consider four domains in  Minigrid~\citep{MinigridMiniworld23}, namely, DoorKey-8x8-v0, MemoryS13Random-v0, MemoryS17Random-v0, and MultiRoom-N4-S5-v0. We defer details on domains and hyper parameter tuning to~\cref{appsec:minigrid} and we report the results multistep action predictions in Appendix~\cref{fig:minigrid-allhorizonsdeeprl}.
\textit{Results:} %
We observe in ~\Cref{fig:minigrid-deeprl} that $\PhiAC$ corresponding to BYOL-AC outperforms the other baselines in 3 out of 4 tasks, $\Phi$ corresponding to BYOL-$\Pi$ is on par with BYOL-AC in 1 out of 4 tasks, whereas $\PhiVARAC$ is a weak baseline and struggles in all 4 tasks. It turns out that for $\PhiVARAC$ corresponding to BYOL-VAR, because the objective is a difference of BYOL-AC and BYOL-$\Pi$, and when we use non-linear predictors (a departure from our theoretical setting)
the optimization behaves similarly to a min-max adversarial-style optimization, where it tries to minimize BYOL-AC and maximizes BYOL-$\Pi$. This means that $\PhiVARAC$ ends up trying to remove features that are good for BYOL-$\Pi$ i.e. $T^{\pi}$, resulting in a worse representation for RL.

\textbf{DQN.}
\label{sec:experiments-off-policy}
Next, we modify DQN's ~\citep{mnih2015human} implementation from \citet{gymnax2022github} by augmenting it with auxiliary losses corresponding to BYOL-$\Pi$ and BYOL-AC in~\Cref{fig:minigrid-deeprl} (because BYOL-VAR performed poorly in Minigrid, we do not evaluate it with DQN). We consider open-AI gym's~\citep{openaigym} classic control environments~\citep{sutton2018reinforcement}, namely, CartPole-v1, MountainCar-v0, and Acrobot-v1. %
\textit{Results:} We report that $\PhiAC$ corresponding to BYOL-AC outperforms $\Phi$ corresponding to BYOL-$\Pi$ in CartPole-v1 (leftmost) and  MountainCar-v0 (center), whereas BYOL-AC performs on par with BYOL-$\Pi$ in Acrobot-v1. %
\begin{figure}[h]
    \centering
    \includegraphics[width=0.28\linewidth]{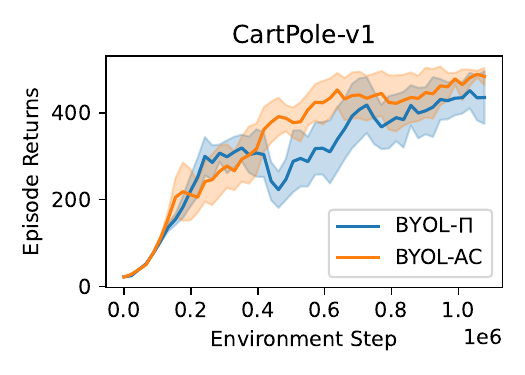}
    \includegraphics[width=0.28\linewidth]{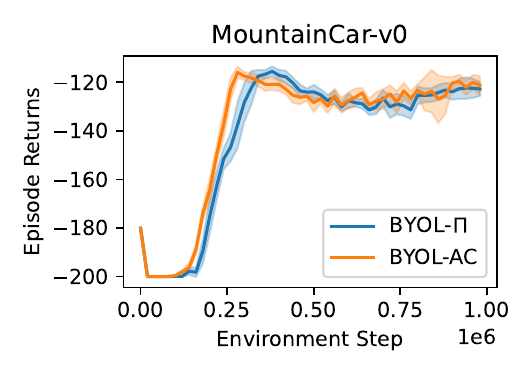}
    \includegraphics[width=0.28\linewidth]{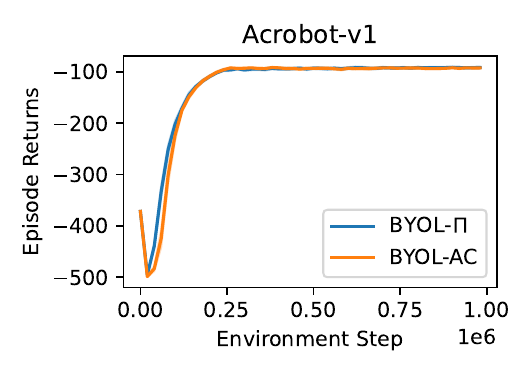}
    \caption{BYOL-AC (orange) is overall better when compared to BYOL-$\Pi$ (blue).}
   \label{fig:dqn-deeprl}
\end{figure} 

\section{Discussion}
\label{sec:discussion}
\textbf{In summary,} we extended previous theoretical analysis to an action-conditional BYOL-AC objective, showing that it learns spectral information about per-action transition dynamics $T_a$. This is in contrast to previously studied BYOL-$\Pi$ which learns spectral information about $T^{\pi}$. We discovered a variance equation that relates BYOL and BYOL-AC (\cref{rem:variance-complete}) in terms of the eigenvalues of $T_a$. This connection results in a novel variance-like objective BYOL-VAR, which learns spectral information pertaining to the residual $(T_a - T^{\pi})$. We unified the three objectives, BYOL-$\Pi$, BYOL-AC, BYOL-VAR, through a model-based lens, showcasing the connection to $T^{\pi}$, $T_a$, and $T_a - T^{\pi}$ respectively. We also unify them through a model-free lens, establishing the connection to learning a certain 1-step value, Q-value, and advantage function respectively. We believe that the variance connection and the two unifying lenses offer unique insights and intuitions into the characteristics of these objectives. The unifying viewpoints facilitate making connections back and forth, such as coming up with a new model-based-like objective, and deriving a corresponding self-predictive ODE objective. 

The \textbf{key takeaway} from our theoretical and empirical investigation is that BYOL-AC is overall a better objective resulting in a better representation $\PhiAC$, when compared to both $\Phi$ and $\PhiVARAC$. \textbf{Future work} could involve relaxing the assumptions in our theoretical analysis, with potential to generalize the theory. Further investigation into BYOL-VAR may also yield new insights and applications, e.g. since BYOL-VAR is concerned with features that distinguish between actions, it may be useful for learning action representations or even option discovery.

\section*{Acknowledgement}
We thank Tom Schaul for providing detailed comments and valuable feedback on a draft of this paper. We also thank Mohammad Gheshlaghi Azar, and Doina Precup for useful discussions.

\bibliography{references}

\newpage
\appendix

\section*{Appendix / Supplemental Material}

\section{Related Work}
\label{sec:relatedwork}

\textbf{Self-Predictive Representation Learning.} At the intersection of representation learning and self-supervised learning, using bootstrapped latent embeddings to train the representations has been an empirically successful approach for both image representation learning \citep{grill2020bootstrap, chen2021exploring}  and reinforcement learning~\citep{guo2020bootstrap}. 

\textbf{Action-Conditioned Predictive Representations.} Action-conditional predictions of the future where the prediction tasks are indicators of events on the finite observation space date back to foundational work introducing predictive state representations (PSRs)~\citep{littman2001predictive}. Deep learning approaches leveraging action-conditional predictions of the future to improve RL performance covers a broad range of ideas. For instance, interleaving an action-dependent (predictor) RNN with an observation dependent RNN~\citep{amos2018learning}, to predicting policies, rewards, values/logits needed for Monte Carlo tree search conditioned on actions~\citep{schrittwieser2020mastering} or options~\citep{oh2015action}. 

\textbf{Understanding Predictive Representations.} A key challenge in self-predictive learning is collapsing solutions.  Both \cite{grill2020bootstrap} and  \cite{guo2020bootstrap} propose BYOL variants where they leverage a target network to inform and train an online network for self-supervised image representation learning and reinforcement learning respectively. While \citet{grill2020bootstrap} posit the need for a  momentum encoder as a key requirement for BYOL to avoid collapsing, \citet{chen2021exploring} show that empirically the
stop-gradient operation is critical to avoid collapse. Both these methods are focused on image representations in iid settings though. With a focus on RL, our work builds upon \cite{tang2022understanding}, who identify and prove conditions for the non-collapse property of self-predictive learning through the lens of a theoretical ODE framework.

\textbf{Spectral Decomposition Lens for Representations in RL.} Recall that we show that BYOL-AC captures the spectral information about action transition matrices, whereas prior work shows that BYOL-$\Pi$ captures spectral information about $P^\pi$. The lens of spectral decomposition of $P^\pi$ or $(I-\gamma P^\pi)^{-1}$, together with eigenvector decomposition~\citep{ferguson2006proto, machado2017eigenoption, lyle2021effect}, and singular value decomposition~\citep{behzadian2019fast, ren2022spectral, chandak2023representations, lan2023bootstrapped} greatly facilitates representation learning research for reinforcement learning.

\section{Additional Experiments and Details}

\begin{figure}[h]
    \centering
    \includegraphics[width=1.0\linewidth]{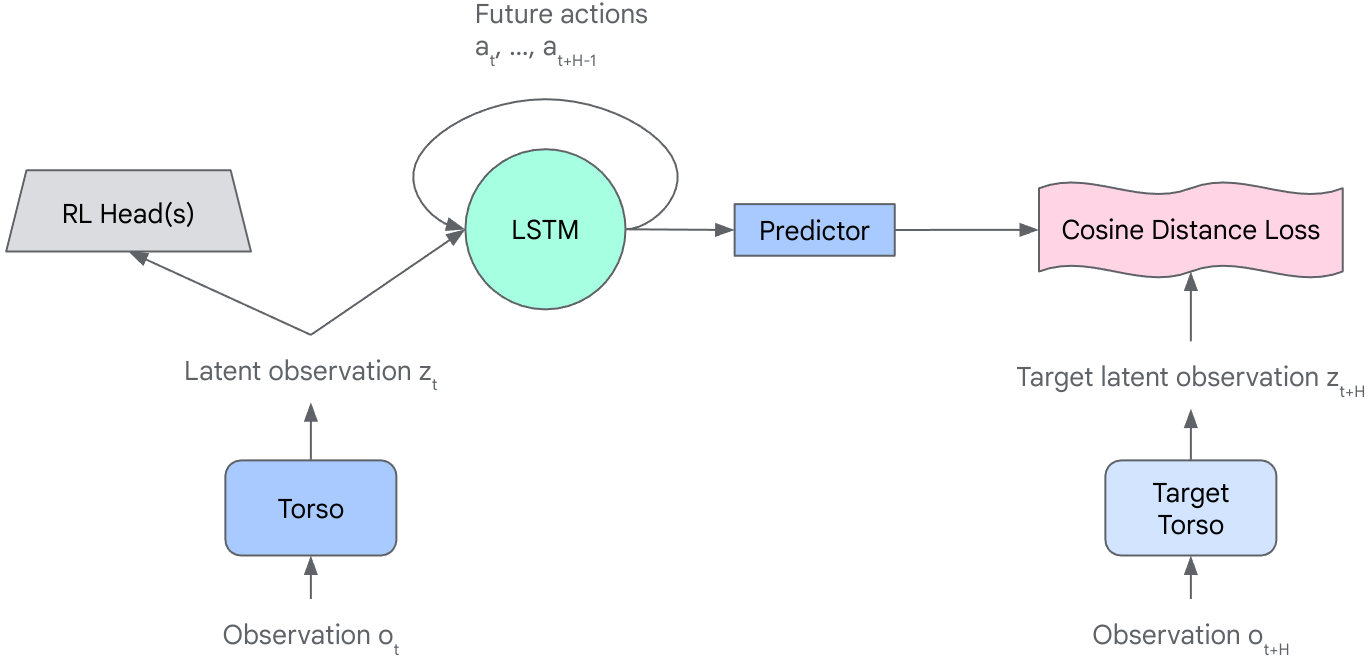}
    \caption{High Level Architecture of our RL Agent. Network details are in~\Cref{appsec:minigrid,appsec:openaigym}}
    \label{fig:networkdiagram}
\end{figure}

\subsection{V-MPO in Minigrid}
\label{appsec:minigrid}
In Minigrid~\citep{MinigridMiniworld23}, we here detail the description of each task: 1) DoorKey-8x8-v0, where the agent must pick up a key in the environment in order to unlock a door and then get to the green goal square, 2) MemoryS13Random-v0, where the agent starts in a small room with a visible object and then has to go through a narrow hallway which ends in a split. At each end of the split there is an object, one of which is the same as the object in the starting room. The agent has to remember the initial object, and go to the matching object at split, 3) MemoryS17Random-v0 is a bigger domain matching the description of the memory test in the previous environment, and 4) MultiRoom-N4-S5-v0, where the agent navigates through a series of connected rooms with doors that must be opened in order to get to the next room. The goal is to reach the final room which has the green square. Note that the 1-3 are fully observable domains, whereas 4 is a partially observable environment.

\begin{figure}[h]
    \centering
    \includegraphics[width=1.0\linewidth]{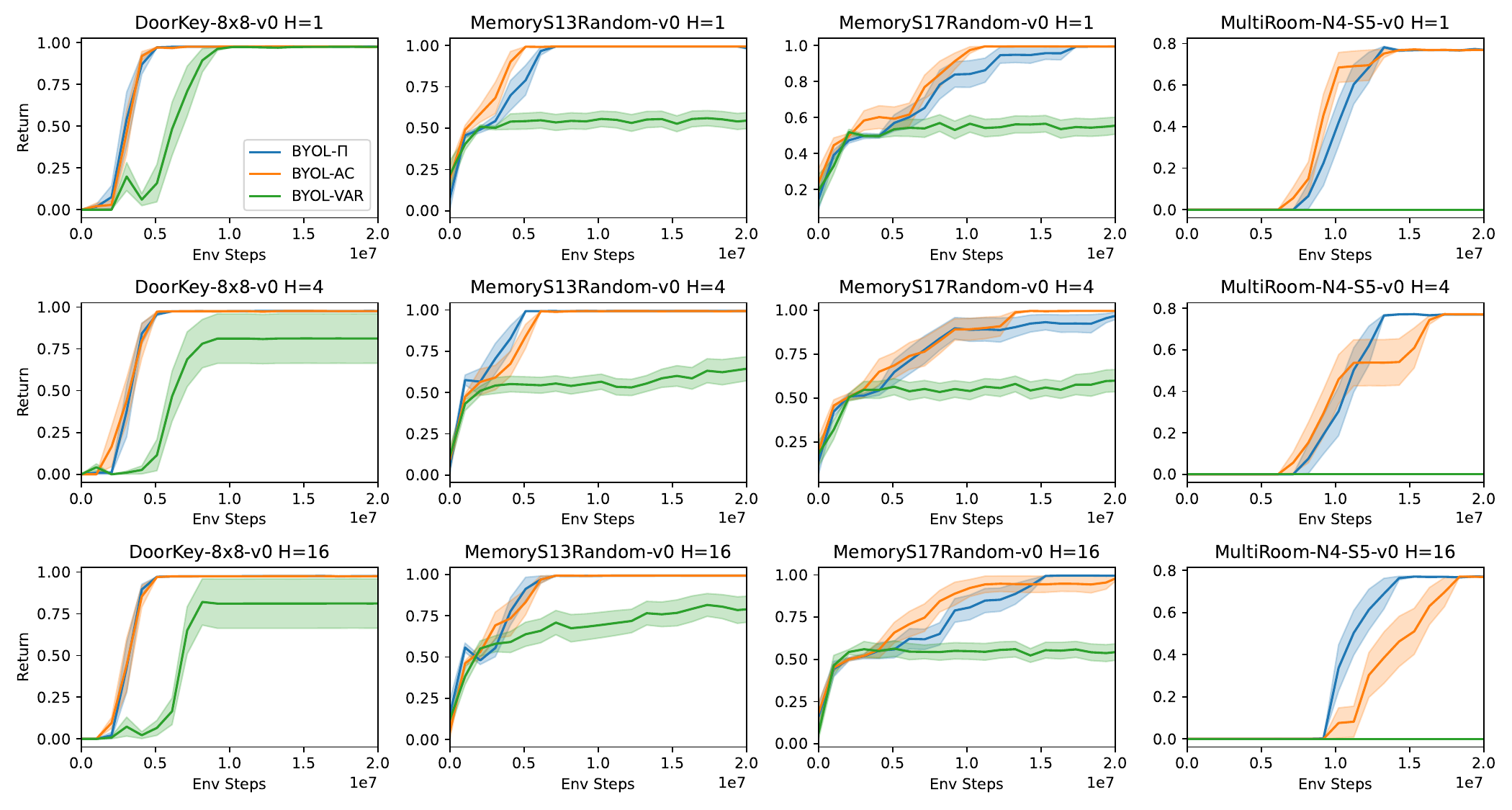}
    \caption{Comparing BYOL-$\Pi$, BYOL-AC, and BYOL-VAR on different domains in Minigrid across varying prediction horizons $H= 1, 4, 16$.}
    \label{fig:minigrid-allhorizonsdeeprl}
\end{figure}

For each baseline considered in Fig.~\ref{fig:minigrid-deeprl} and Fig~\ref{fig:minigrid-allhorizonsdeeprl}, we first tuned the hyper-parameters of the base RL algorithm i.e. V-MPO in this case. We then tuned the BYOL-$\Pi$ baseline, followed by running both BYOL-AC and BYOL-VAR for various horizons [H].

Our high level architecture is shown in~\Cref{fig:networkdiagram}. Here are the network details:
\begin{itemize}
    \item Torso: small ResNet
    \begin{enumerate}
        \item Conv2D[channel=32,stride=2,kernel=3]
        \item ResBlock[channel=32,stride=1,kernel=3]
        \item ResBlock[channel=32,stride=1,kernel=3]
        \item Conv2D[channel=128,stride=2,kernel=3]
        \item ResBlock[channel=128,stride=1,kernel=3]
        \item ResBlock[channel=128,stride=1,kernel=3]
        \item Conv2D[channel=256,stride=2,kernel=3]
        \item ResBlock[channel=256,stride=1,kernel=3]
        \item ResBlock[channel=256,stride=1,kernel=3]
        \item Flatten
        \item Linear[256]
    \end{enumerate}
    \item LSTM: Hidden size 256
    \item Predictor: MLP[128, 256, 512, 256]
    \item RL Head
    \begin{itemize}
        \item Value Head: MLP[512, 1]
        \item Policy Head: MLP[512, 7]
    \end{itemize}
\end{itemize}
And here are our hyperparameters.
\begin{itemize}
    \item Minibatch batch size: 48
    \item Minibatch sequence length: 30
    \item Adam Optimizer[learning\_rate=1e-4, b1=0.9, b2=0.999, eps=1e-8, eps\_root=1e-8]
\end{itemize}

\subsection{DQN in Open-AI Gym}
\label{appsec:openaigym}
For open-AI gym's~\citep{openaigym} classical domains~\citep{sutton2018reinforcement}, we consider: 1) Cartpole, where the a pole is placed upright on the cart and the goal is to balance the pole by applying forces in the left and right direction on the cart, 2) Acrobot, where the goal is to apply torques on an actuated joint to swing the free end of the linear chain above a given height while starting from the initial state of hanging downwards, and 3) Mountain Car, where a car is placed stochastically at the bottom of a sinusoidal valley and the goal is to accelerate the car to reach the goal state on top of the right hill.

For both BYOL-$\Pi$ and BYOL-AC here, we first tuned and fixed all the DQN specific parameters and then tuned to obtain the the best hyper-parameters for both methods. 

Our high level architecture is shown in~\Cref{fig:networkdiagram}. For these domains, we use a soft-dicretization of the observation space before passing it into our torso. We independently convert each observation dimension to a soft-one-hot encoding using a Gaussian distribution. For example, if we want to encode the value $a=0.2$ into a soft-one-hot with 11 bins, with lower bound 0 and upper bound 1, we first compute the distance between 0.2 and the 11 bin points ($b_i$) (0, 0.1, 0.2, $\dots$, 0.9, 1.0) using $e^{-\frac{1}{2} \left( \frac{a - b_i}{\sigma} \right)^2}$. Then we normalize so that this sums to one.
\paragraph{CartPole Network:}
\begin{itemize}
    \item Torso: MLP[128, 64]
    \item LSTM: Hidden size 256
    \item Predictor: MLP[256, 64]
    \item RL Head
    \begin{itemize}
        \item Q Head: MLP[128, 2]
    \end{itemize}
\end{itemize}
And here are our hyperparameters.
\begin{itemize}
    \item Minibatch batch size: 8
    \item Minibatch sequence length: 16
    \item Replay size 1e5
    \item Adam Optimizer[learning\_rate=1e-4]
    \item Soft Discretization
    \begin{itemize}
        \item 64 Bins
        \item $\sigma$ 0.1
        \item Observation lower bound [-5, -5, -0.5, -5]
        \item Observation lower bound [5, 5, 0.5, 5]
    \end{itemize}
\end{itemize}

\paragraph{MountainCar Network:}
\begin{itemize}
    \item Torso: MLP[64, 64]
    \item LSTM: Hidden size 256
    \item Predictor: MLP[256, 64]
    \item RL Head
    \begin{itemize}
        \item Q Head: MLP[256, 256, 2]
    \end{itemize}
\end{itemize}
And here are our hyperparameters.
\begin{itemize}
    \item Minibatch batch size: 8
    \item Minibatch sequence length: 16
    \item Replay size 1e5
    \item Adam Optimizer[learning\_rate=5e-4]
    \item Soft Discretization
    \begin{itemize}
        \item 32 Bins
        \item $\sigma$ 0.05
        \item Observation lower bound [-1.2, -0.07]
        \item Observation lower bound [0.6, 0.07]
    \end{itemize}
\end{itemize}

\paragraph{Acrobot Network:}
\begin{itemize}
    \item Torso: MLP[128, 64]
    \item LSTM: Hidden size 256
    \item Predictor: MLP[256, 64]
    \item RL Head
    \begin{itemize}
        \item Q Head: MLP[256, 256, 3]
    \end{itemize}
\end{itemize}
And here are our hyperparameters.
\begin{itemize}
    \item Minibatch batch size: 8
    \item Minibatch sequence length: 16
    \item Replay size 1e5
    \item Adam Optimizer[learning\_rate=5e-4]
    \item Soft Discretization
    \begin{itemize}
        \item 32 Bins
        \item $\sigma$ 0.1
        \item Observation lower bound [-1., -1., -1., -1., -12.57, -28.27]
        \item Observation lower bound [1., 1., 1., 1., 12.57, 28.27]
    \end{itemize}
\end{itemize}

\subsection{Computing and Libraries}
\label{appsec:libraries-and-compute}
All experiments were conducted using either TPU-V2 or L4 GPU instances. Libraries that enabled this work, include NumPy~\citep{oliphant2006guide}, SciPy~\citep{virtanen2020scipy}, Matplotlib~\citep{hunter2007matplotlib}, JAX~\citep{deepmind2020jax}, Gymnax~\citep{gymnax2022github}, Flashbax~\citep{flashbax}, and Flax~\citep{flax2020github}.

\section{Measuring robustness of $\Phi$ to changes in policy}
\label{appsec:robustnessofphi-to-changes}
Having examined how well a representation fits to different objectives under the same policy in Table~\ref{tab:value-mse}, we now investigate how robust each representation is to off-policy data. To establish a measure of robustness in the learned representation, we propose examining the distance in the representations with respect to the changes in policy $\pi$ and therefore to changes in the induced dynamics $P$ as follows.  %
Formally, for a given policy $\pi'$ obtained by perturbing $\pi$, we define
\begin{align*}
    \Delta(\Phi) &:= d_{\text{Gr}}(\Phi^*_\pi, \Phi^*_{\pi'}) \, ,
\end{align*}
where $\Phi^*_{\pi}, \Phi^*_{\pi'}$ are limit points of the dynamics for the ODEs under $\pi$, $\pi'$ respectively, with the same initialisation of $\Phi_0$, and $d_{\text{Gr}}$ is the Grassmann distance.

\begin{table}[h]
\caption{\textbf{Stability Analysis.} For each method, we report the $\Delta$ in the representation upon perturbation in the initial policy, and $P()$ denotes the probability of a method with minimal shift in the representation compared to the other two representations. We report the standard error in the bracket corresponding to 200 independent runs over randomly initialized policy to run the ODE to obtain $\Phi$.}
\label{tab:stabilityexperiments}
\centering
\begin{adjustbox}{width=\columnwidth,center}
\begin{tabular}{l|ll|ll|ll}
\hline
\cellcolor{aliceblue} \textbf{Method}   & \cellcolor{aliceblue} & \cellcolor{aliceblue} \textbf{BYOL-$\Pi$} [$\Phi$]  &   \cellcolor{aliceblue} \textbf{BYOL-AC} [$\PhiAC$]    & \cellcolor{aliceblue}  & \cellcolor{aliceblue} \textbf{BYOL-VAR } [$\PhiVARAC$] &   \cellcolor{aliceblue}\\
\hline
\textbf{Initial Policy}    & $\Delta(\Phi)$ &  \textbf{P($\Delta(\Phi) \leq (\Delta(\PhiAC), \Delta(\PhiVARAC)) $)} &   $\Delta(\PhiAC) $ &  \textbf{P($\Delta(\PhiAC) \leq (\Delta(\Phi), \Delta(\PhiVARAC)$)} &  $\Delta(\PhiVARAC)$ & \textbf{P($\PhiVARAC \leq (\Delta(\Phi), \Delta(\PhiAC)$)} \\ \hline
$\epsilon$-greedy ($\epsilon=0.01$)  & 0.032 (0.004) & 0.11  & 0.023 (0.007) &  \cellcolor{lavender} \textbf{0.89}  & 0.741 (0.03) & 0.0   \\
$\epsilon$-greedy ($\epsilon=0.03$)  & 0.042 (0.008)  & 0.095 & 0.014 (0.001) &  \cellcolor{lavender} \textbf{0.905} & 0.54 (0.03) & 0.0  \\
$\epsilon$-greedy ($\epsilon=0.1$)   & 0.037 (0.008)  & 0.095 & 0.027 (0.007) &  \cellcolor{lavender} \textbf{0.84}  & 0.243 (0.025) & 0.065  \\
$\epsilon$-greedy ($\epsilon=0.25$)  & 0.043 (0.008)  & 0.08  & 0.035 (0.009) &  \cellcolor{lavender} \textbf{0.76}  & 0.167 (0.02) & 0.16     \\
\hline
\end{tabular}

\end{adjustbox}
\end{table}

\section{Additional Intuition on Theory}
\label{sec:intuition_theory}

\subsection{Intuition and Implications of Assumptions.}

\restateassortho*
This assumption on the orthogonal initialization is relatively reasonable and often considered in deep RL methods as well. A random matrix where entries are taken from a unit normal distribution is highly likely to be close to orthogonal. This also suggests that a randomly initialized neural network may have a good chance to approximately satisfy this assumption depending on the input type.

\restateassuniformstate*
This uniform state assumption doesn't often hold in practice, and unfortunately is important for our proofs to hold. We have attempted to relax this assumption, but it requires considerable amount of work, and therefore is out of scope of this paper.

\restateasssymmetrict*
We note that while our theoretical guarantees might require the transition matrices to be symmetric, \citet{tang2022understanding} have shown it is possible to relax this assumption. To do so, we can consider doubly stochastic matrices, together with a bi-directional algorithm that models the backward transition process in addition to the forward transition dynamics. Leveraging this insight, all our results would still hold for non-symmetric transition dynamics. We here demonstrate the evolution of trace objective with time when the assumption of symmetric MDPs is relaxed. The numerical evidence here shows that the learning dynamics can still capture useful spectral information about the corresponding transition dynamics. Notably, the assumption is stricter for BYOL-VAR which shows the trace objective might not be capturing much useful information when this assumption is relaxed.

\begin{figure}[h!]
    \centering
    \includegraphics[width=1.0\linewidth]{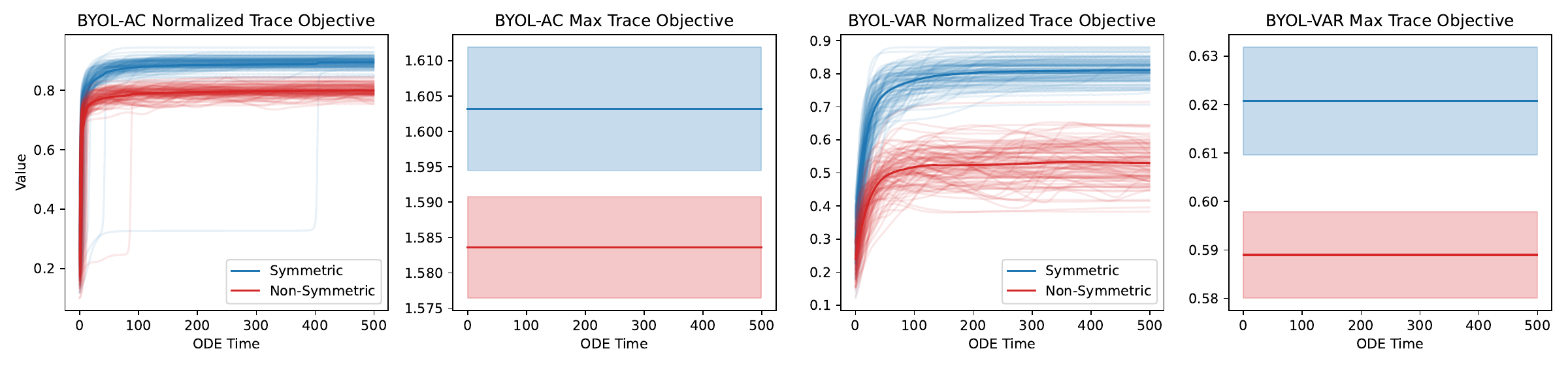}
    \caption{\textbf{Ratio between the trace objective and the value of the objective for the top k eigenvectors of $T_a$ and $T_a - T^\pi$} corresponding to BYOL-$\Pi$ and BYOL-AC respectively, versus the number of ODE training iterations. The light light curve corresponds to one of 100 independent runs over randomly generated MDPs, and the solid curve shows the median over runs.}
    \label{fig:tracesplot}
\end{figure}

\restateasssymmetricta*
Analogous to requiring the policy induced transition matrix $T^\pi$ to be symmetric, we also require the per-action matrices $T_a$ to be symmetric, which is a strong and impractical assumption. If this assumption is violated, the trace objective is no longer a Lyapunov function. We did a simple investigation where we ran the ODE without this assumption (\Cref{fig:tracesplot}) and it seems that even then the ODE does end up increasing the trace objective on average. One future avenue, where this assumption could potentially yield more interesting insights is the hierarchical RL setting. When considering options (multi-step action), this assumption can be relaxed for the primitive actions while only requiring the option-level transition matrix $T_o$ to be symmetric.

\restateasscommonta*

This common eigenspace assumption is mainly used to exactly specify what the maximizers of the trace objectives look like. Actually, the negative trace objectives being Lyapunov functions does not require this assumption. Our linear experiments in~\Cref{tab:unifyingmodelfree-linearexp} do not satisfy this assumption. Without this assumption though it becomes very difficult to specify what the maximizer looks like, and whether it is a critical point of the ODE. Thus while we make use of this assumption in our theory, it is not strictly necessary and is mainly used to obtain easily interpretable results. Notably, our key results encompassing the two unifying views do not require this assumption, and while the variance equation does require it, the intuition behind the variance equation still holds.

\restateassuniformpolicy*
We note that the assumption~\ref{ass:uniform-policy} is more strict than necessary. It is enough for the policy $\pi$ to be state-independent, i.e. the same distribution over actions at every state, and all the results would still hold with just being $\pi$-weighted. However we assume this stronger version to keep the theory simple and still retain all the important insights that we will analyze.
The choice of a uniform policy is meant to simplify the presentation. 

We here provide the objective with a different distribution $\pi$ over the actions to obtain similar results. For example, with $\pi$ instead of a uniform distribution, would result in:
\[
    f_{\text{BYOL-AC}}(\Phi) = \textstyle \sum_a \pi_a \mtrace \left(\Phi^T T_a \Phi \Phi^T T_a \Phi \right)
\]
and
\[
     f_{\text{BYOL-VAR}}(\Phi) = \sum_a \pi_a \mtrace \left(\Phi^T T_a \Phi \Phi^T T_a \Phi\right)  - \mtrace \left(\Phi^T T^{\pi} \Phi \Phi^T T^{\pi} \Phi  \right).
\]
Evidently, changing $\pi$ will accordingly change the maximizers of $f_{\text{BYOL-AC}}$ and $f_{\text{BYOL-VAR}}$, as well as the outcome of representation learning.

\subsection{Intuition Behind the Loss and Trace Objective.}

\textbf{On convergence:} Our main~\Cref{thm:ode-byol,thm:ode-byol-ac,thm:ode-byol-var} focus on analyzing the maximizer to the trace objectives in~\Cref{lem:trace-byol,lem:trace-byol-ac,lem:trace-byol-var}. We show that the maximizer of the trace objective is indeed a critical point of the corresponding ODE. However there can exist other, non-maximizer, critical points of the ODE. This means it is possible that the ODE converges to a sub-optimal critical point. Because we are interested in the optimal point, we don't dive further into the properties of the sub-optimal critical points. However, it can be a useful direction for future work to investigate whether the sub-optimal critical points are stable or unstable, and if there is a way to guarantee convergence to the maximum.

\section{Proofs}
\label{sec:app-proofs}

The Assumptions, Lemmas, and Theorems of~\Cref{sec:ode-byol} are taken (with some wording changes to be consistent with this work) from the results in~\citet{tang2022understanding}.

\subsection{Proofs of~\Cref{sec:byol-ac}: BYOL-AC}

We first present and prove a few helper lemmas about what are $P_a^*$ and $\dot{\Phi}$. Note as shorthand we let $\E[xx^T] = D_X$, which means $\E_{y \sim p(\cdot \mid x, a)}[xy^T] = D_X T_a$.

\subsubsection{Finding $P_a^*$}

\begin{restatable}[Optimal $P_a^*$]{relemma}{restatelempastar}
\label{lem:optimal-pa}
We have the following.
\begin{align*}
    P_a^* &\in \argmin_{P_a}  \E \Bigg[ \big| \big| P_a^T \Phi^T x - \sg(\Phi^T y)  \big| \big|^{2}_{2} \Bigg] \\
    \implies & \left( \Phi^T D_X \Phi \right) P_a^* = \Phi^T D_X T_a \Phi
\end{align*}
\end{restatable}
\begin{proof}
We first expand and rewrite the objective as a trace objective. Note that we can ignore the stop-gradient because we are only concerned with $P_a$.
{\allowdisplaybreaks
\begin{align*}
    & \E \Bigg[ \big| \big| P_a^T \Phi^T x - \Phi^T y  \big| \big|^{2}_{2} \Bigg] \\
    &= \E_{x, a, y \mid (x, a)} \left[ x^T \Phi P_a P_a^T \Phi^T x - 2 y^T \Phi P_a^T \Phi^T x + y^T \Phi \Phi^T y  \right] \\
    &= \frac{1}{|A|} \sum_a \E_{x, y \mid (x, a)} \left[ x^T \Phi P_a P_a^T \Phi^T x - 2 y^T \Phi P_a^T \Phi^T x + y^T \Phi \Phi^T y  \right] \quad \text{since $\pi$ is uniform} \\
    &= \frac{1}{|A|} \sum_a \E_{x, y \mid (x, a)} \left[ \mtrace \left( x^T \Phi P_a P_a^T \Phi^T x - 2 y^T \Phi P_a^T \Phi^T x + y^T \Phi \Phi^T y \right) \right] \\
    &= \frac{1}{|A|} \sum_a  \mtrace \left( \E[xx^T] \Phi P_a P_a^T \Phi^T - 2 \E[xy^T] \Phi P_a^T \Phi^T + \E[yy^T] \Phi \Phi^T \right) 
\end{align*}
}
The $\E[yy^T]$ term can be considered as just a constant since it does not depend on $P_a$. Thus we end up with
\begin{align*}
    &=  \frac{1}{|A|} \sum_a  \mtrace \left( D_X \Phi P_a P_a^T \Phi^T - 2 D_X T_a \Phi P_a^T \Phi^T \right) + \mathrm{Constant}
\end{align*}
Next, we take the derivative w.r.t. $P_a$.
\begin{align*}
    \frac{\partial}{\partial P_a} \;  \left( \frac{1}{|A|} \sum_a  \mtrace \left( D_X \Phi P_a P_a^T \Phi^T - 2 D_X T_a \Phi P_a^T \Phi^T \right) \right) &= \frac{1}{|A|} \left( 2 \Phi^T D_X \Phi P_a - 2 \Phi^T D_X T_a \Phi \right)
\end{align*}
Finally we use the fact that the derivative is zero for $P_a^*$:
\begin{align*}
    0 &= \frac{1}{|A|} \left( 2 \Phi^T D_X \Phi P_a - 2 \Phi^T D_X T_a \Phi \right) \\
    \implies \left( \Phi^T D_X \Phi \right) P_a^* &= \Phi^T D_X T_a \Phi
\end{align*}
\end{proof}

\subsubsection{Computing $\dot{\Phi}$}

\begin{restatable}{relemma}{restatelemphidotac}
\label{lem:phi-dot-ac}
$\dot{\Phi}$ satisfies
\begin{align*}
    \dot{\Phi} &= - \nabla_{\Phi} \mathbb{E} \Bigg[ \big| \big| (P_a)^{T} \Phi^{T} x - \sg(\Phi^{T} y)  \big| \big|^{2}_{2} \Bigg] \Bigg\vert_{P_a = P_a^{*}} \\
    &= - \frac{2}{|A|} \sum_a \left( D_X \Phi P_a^* - D_X T_a \Phi \right) (P_a^*)^T
\end{align*}
\end{restatable}
\begin{proof}
We first expand out the objective into a trace objective. The steps are similar to the steps above in the proof for~\Cref{lem:optimal-pa}.
\begin{align*}
    & \mathbb{E} \Bigg[ \big| \big| (P_a)^{T} \Phi^{T} x - \sg(\Phi^{T} y)  \big| \big|^{2}_{2} \Bigg] \\
    &= \frac{1}{|A|} \sum_a  \mtrace \left( D_X \Phi P_a P_a^T \Phi^T   - 2 D_X T_a \sg(\Phi) P_a^T \Phi^T + \E[yy^T] \sg(\Phi \Phi^T) \right) \\
    &= \frac{1}{|A|} \sum_a  \mtrace \left( D_X \Phi P_a P_a^T \Phi^T  - 2 D_X T_a \sg(\Phi) P_a^T \Phi^T \right) + \mathrm{Constant}
\end{align*}
Next we inspect the derivative:
\begin{align*}
    & \frac{\partial}{\partial \Phi} \left( \frac{1}{|A|} \sum_a  \mtrace \left( D_X \Phi P_a P_a^T \Phi^T  - 2 D_X T_a \sg(\Phi) P_a^T \Phi^T \right) \right) \\
    &= \frac{2}{|A|} \sum_a \left( D_X \Phi P_a - D_X T_a \Phi \right) P_a^T
\end{align*}
Finally, we plug in $P_a^*$ to obtain $\dot{\Phi}$:
\begin{align*}
    \dot{\Phi} &= - \frac{2}{|A|} \sum_a \left( D_X \Phi P_a^* - D_X T_a \Phi \right) (P_a^*)^T
\end{align*}
\end{proof}

\subsubsection{Proof of~\Cref{lem:non-collapse-byol-ac} and Simplified $P_a^*$ and $\dot{\Phi}$}

\restatelemnoncollapseac*
\begin{proof}
\begin{align*}
    \Phi^T \dot{\Phi} &= \Phi^T \left( - \frac{2}{|A|} \sum_a \left( D_X \Phi P_a^* - D_X T_a \Phi \right) (P_a^*)^T \right) \\
    &= \left( - \frac{2}{|A|} \sum_a \left( \underbrace{\Phi^T D_X \Phi P_a^* - \Phi^TD_X T_a \Phi}_{=0\text{ from definition of $P_a^*$}}  \right) (P_a^*)^T \right) \\
    &= 0.
\end{align*}
Therefore, $\frac{\mathrm{d}}{\mathrm{d}t}\Phi_t^T \Phi_t = 0$, that is, $\Phi_t^T \Phi_t$ is a constant for all $t$.
Since $\Phi_0^T \Phi_0 = I$ by \cref{ass:orthogonal-init}, the result follows.
\end{proof}

As a consequence, and in combination with~\Cref{ass:orthogonal-init,ass:uniform-state,ass:symmetric-t,ass:uniform-policy,ass:symmetric-ta}, we can simplify our expressions for $P_a^*$ and $\dot{\Phi}$. Note that with a uniform distribution for $D_X$, we have $D_X = |\mathcal{X}|^{-1} I$.
\begin{align}
    P_a^* &= \Phi^T T_a \Phi \label{eq:optimal-pa-simplified}\\
    \dot{\Phi} &= \left( I - \Phi \Phi^T \right) \left( \frac{2}{|A||\mathcal{X}|} \sum_a  T_a \Phi \Phi^T T_a \Phi \right) \label{eq:phi-dot-ac-simplified}
\end{align}

\subsubsection{Proof of~\Cref{lem:trace-byol-ac}}

\restatelemtraceac*
\begin{proof}
To check that $-f_{\text{BYOL-AC}}$ is a Lyapunov function for the BYOL-AC ODE, we will verify that its time derivative is strictly negative (it is strictly decreasing) for all non critical points (a critical point being $\Phi$ where $\Phi_t = \Phi \Rightarrow \dot{\Phi} = 0$). By chain rule through trace we have
\begin{align*}
    \frac{\mathrm{d}}{\mathrm{d}t} \left( -f_{\text{BYOL-AC}}(\Phi_t) \right) &= -\mtrace \left( \frac{\partial}{\partial \Phi} \left( \frac{1}{|A|} \sum_a \Phi^T T_a \Phi \Phi^T T_a \Phi \right)^T \cdot \dot{\Phi} \right) \\
    &= -\mtrace \left( \left( \frac{1}{|A|} \sum_a \Phi^T T_a \Phi \Phi^T T_a \right) \cdot \dot{\Phi} \right) \\
    &= -\mtrace \left( \left( \frac{1}{|A|} \sum_a \Phi^T T_a \Phi \Phi^T T_a \right) \cdot \left( I - \Phi \Phi^T \right) \left( \frac{2}{|A||\mathcal{X}|} \sum_a  T_a \Phi \Phi^T T_a \Phi \right) \right)
\end{align*}
Since $\Phi$ is orthogonal, $(I - \Phi\Phi^T)$ is a projection matrix i.e. $(I - \Phi\Phi^T)(I - \Phi\Phi^T) = (I - \Phi\Phi^T)$. So we can add an extra projection in.
\begin{align*}
    &= -\mtrace \left( \left( \frac{1}{|A|} \sum_a \Phi^T T_a \Phi \Phi^T T_a \right) \left( I - \Phi \Phi^T \right) \cdot \left( I - \Phi \Phi^T \right) \left( \frac{2}{|A||\mathcal{X}|} \sum_a  T_a \Phi \Phi^T T_a \Phi \right) \right) \\
    &= -\frac{|\mathcal{X}|}{2} \mtrace \left( \dot{\Phi}^T \dot{\Phi} \right),
\end{align*}
which is strictly negative when $\dot{\Phi} \neq 0$.
\end{proof}

\subsubsection{Helper Lemmas for the Proof of~\Cref{thm:ode-byol-ac}}

Before we prove~\Cref{thm:ode-byol-ac}, we need to present two more helpful Lemmas. The first is the well-known Von Neumann trace inequality. The second concerns maximizing a particular constrained trace expression.

\begin{restatable}[Von Neumann Trace Inequality]{relemma}{restatelemvon}
\label{lem:von}
Let $A,B \in \real^{n \times n}$ with singular values $\alpha_1 \geq \alpha_2 \geq \dots \geq \alpha_n$ and $\beta_1 \geq \beta_2 \geq \dots \geq \beta_n$ respectively. Then
\begin{align*}
    \mtrace \left( AB \right) \leq \sum_{i=1}^n \alpha_i \beta_i
\end{align*}
\end{restatable}

\begin{restatable}[Maximizer of Constrained Trace Expression]{relemma}{restatelemmaxtrace}
\label{lem:max-trace}
Let $B_a \in \real^{n \times n}$ be symmetric matrices for $a \in \{1, 2, \dots, |A| \}$ and $\Phi \in \real^{n \times k}$ where $n \geq k$. Assume all $B_a$ share the same eigenvectors, so $B_a = Q D_a Q^T$ is an eigendecomposition of $B_a$ and $D_a$ is a diagonal matrix of eigenvalues $\beta_{a,1}, \beta_{a,2}, \dots, \beta_{a,n}$. Let $Q = [Q_k \; \overline{Q}_k]$ where $Q_k \in \real^{n \times k}$ are the first $k$ columns of $Q$ (and $\overline{Q}_k$ are the rest of the columns). Let the eigenvectors in $Q$ be sorted from largest to smallest according to $\frac{1}{|A|} \sum_a D_a^2$ i.e. $Q_k$ is the top-$k$ eigenvectors for $\frac{1}{|A|} \sum_a B_a^2$. Then under the constraint that $\Phi$ is orthogonal ($\Phi^T \Phi = I$), for any orthogonal matrix $C \in \real^{k \times k}$, we have
\begin{align*}
    Q_k C \in \argmax_{\Phi} \; \frac{1}{|A|} \sum_a \mtrace \left( \Phi^T B_a \Phi \Phi^T B_a \Phi \right).
\end{align*}
In other words, a matrix whose columns span the same subspace as the columns of $Q_k$ is a maximizer of the constrained trace expression.
\end{restatable}
\begin{proof}
We first establish an upper bound on the trace expression.
\begin{align*}
    0 &\leq \frac{1}{|A|} \sum_a  \mtrace \left(  \underbrace{\Phi^T B_a (I - \Phi \Phi^T) (I - \Phi \Phi^T) B_a \Phi }_{\text{positive semi-definite}} \right) \\
    &= \frac{1}{|A|} \sum_a  \mtrace \left( \Phi^T B_a \underbrace{(I - \Phi \Phi^T)}_{\mathclap{\text{projection matrix}}} B_a \Phi \right)
\end{align*}
where $(I - \Phi^T \Phi)$ is a projection since $\Phi$ is orthogonal. Then we have
\begin{align*}
    & \frac{1}{|A|} \sum_a \mtrace \left( \Phi^T B_a \Phi \Phi^T B_a \Phi \right) \\
    &\leq \frac{1}{|A|} \sum_a \mtrace \left( \Phi^T B_a \Phi \Phi^T B_a \Phi \right) + \frac{1}{|A|} \sum_a  \mtrace \left( \Phi^T B_a (I - \Phi \Phi^T) B_a \Phi \right) \\
    &= \frac{1}{|A|} \sum_a \mtrace \left( \Phi^T B_a^2 \Phi \right) \\
    &= \mtrace \left( \left( \frac{1}{|A|} \sum_a B_a^2 \right)  \Phi \Phi^T  \right) \quad \text{cyclic property of trace}
\end{align*}
Since $\Phi$ is orthogonal, we also have that $\Phi \Phi^T$ is a projection matrix. $\Phi \Phi^T$ is also symmetric, which means that it has an eigendecomposition. Due to being a projection, its eigenvalues must either be $0$ or $1$. More specifically, since $\Phi$ has rank $k$, it must be that $k$ of its eigenvalues $1$ and the rest are $0$. Then by the Von Neumann Trace Inequality (\Cref{lem:von}), we can bound
\begin{align*}
    \mtrace \left( \left( \frac{1}{|A|} \sum_a B_a^2 \right) \Phi \Phi^T  \right) &\leq \sum_{i=1}^k \left( \frac{1}{|A|} \sum_a \beta_{a, i}^2 \right) 
\end{align*}
i.e. it is bounded above by the sum of the top-$k$ eigenvalues of $\left( \frac{1}{|A|} \sum_a B_a^2 \right)$, since the eigenvalues of $\Phi\Phi^T$ essentially act as a filter, picking out $k$ eigenvalues of $\left( \frac{1}{|A|} \sum_a B_a^2 \right)$. Thus, to summarize what we have so far, we have the following upper bound
\begin{align}
    \frac{1}{|A|} \sum_a \mtrace \left( \Phi^T B_a \Phi \Phi^T B_a \Phi \right) \leq \sum_{i=1}^k \left( \frac{1}{|A|} \sum_a \beta_{a, i}^2 \right) 
\end{align}
Note that is a global upper bound since it holds for any $\Phi$. Next we will show that $\Phi = Q_k C$ actually attains the upper bound, and is thus a maximizer. We plug this into the trace expression.
{\allowdisplaybreaks
\begin{align*}
    & \frac{1}{|A|} \sum_a \mtrace \left( (Q_k C)^T B_a (Q_k C) (Q_k C)^T B_a (Q_k C) \right) \\
    &= \frac{1}{|A|} \sum_a \mtrace \left( Q_k^T B_a Q_k Q_k^T B_a Q_k \right) \\
    &= \frac{1}{|A|} \sum_a \mtrace \left( Q_k^T (Q D_a Q^T) Q_k Q_k^T (Q D_a Q^T) Q_k \right) \qquad \text{substitute eigendecomposition of $A$} \\
    &= \frac{1}{|A|} \sum_a \mtrace \left( [I_k \; 0] D_a [I_k \; 0]^T [I_k \; 0] D_a [I_k \; 0]^T \right) \qquad \text{since} \; Q_k^T Q = Q_k^T [Q_k \; \overline{Q}_k] = [I_k \; 0] \\
    &= \mtrace \left( [I_k \; 0]^T [I_k \; 0] \left( \frac{1}{|A|} \sum_a D_a^2\right) \right) \qquad \text{since diagonal matrices commute}\\
    &= \sum_{i=1}^k \left( \frac{1}{|A|} \sum_a \beta_{a, i}^2 \right)
\end{align*}
}
Thus $\Phi = Q_k C$ results in the trace expression attaining its global upper bound, so it is a maximizer.
\end{proof}

\subsubsection{Proof of~\Cref{thm:ode-byol-ac}}

\restatethmodeac*
\begin{proof}

Let the eigendecomposition of $T_a$ be $T_a = Q D_a Q^T$ for all actions $a$. By \Cref{ass:common-ta}, $Q$ are the shared eigenvectors across all $T_a$. Let the eigenvectors of $Q$ be sorted from largest to smallest according to $\frac{1}{|A|} \sum_a D_a^2$. Let $Q = [Q_k \; \overline{Q}_k]$ so that $Q_k$ are the top-$k$ eigenvectors. 
Let $C \in \real^{k \times k}$ be an arbitrary orthogonal matrix.
From~\Cref{lem:max-trace} we have that $\PhiAC^* = Q_k C$ is a maximizer of the trace objective. Because we are mutiplying $Q_k$ by $C$ on the right, $\PhiAC^*$ columns span the same subspace as the columns of $Q_k$ i.e. the same subspace as the top-$k$ eigenvectors of $\left( \frac{1}{|A|} \sum_a T_a^2 \right)$, as we wanted to show.

To finish the proof, it remains to show that $\PhiAC^*$ is a critical point of the BYOL-AC ODE, that is, $\Phi_t = \PhiAC^* \Rightarrow \dot{\Phi} = 0$.
We plug $\PhiAC^* = Q_k C$ into $\dot{\Phi}$ (\Cref{eq:phi-dot-ac-simplified}):
{\allowdisplaybreaks
\begin{align*}
    & \left( I - (Q_k C) (Q_k C)^T \right) \left( \frac{2}{|A||\mathcal{X}|} \sum_a  T_a (Q_k C) (Q_k C)^T T_a (Q_k C) \right) \\
    &= \left( I - Q_k Q_k^T \right) \left( \frac{2}{|A||\mathcal{X}|} \sum_a  T_a Q_k Q_k^T T_a Q_k C \right) \\
    &= \left( I - Q_k Q_k^T \right) \left( \frac{2}{|A||\mathcal{X}|} \sum_a  (Q D_a Q^T) Q_k Q_k^T (Q D_a Q^T) Q_k C \right) \quad \text{substitute $T_a = Q D_a Q^T$} \\
    &= \frac{2}{|A||\mathcal{X}|} \sum_a \left( I - Q_k Q_k^T \right)  Q D_a [I_k \; 0]^T [I_k \; 0] D_a [I_k \; 0]^T C \quad \text{where } Q_k^T Q = Q_k^T [Q_k \; \overline{Q}_k] = [I_k \; 0] \\
    &= \frac{2}{|A||\mathcal{X}|} \sum_a \left( I - Q_k Q_k^T \right)  Q  [I_k \; 0]^T [I_k \; 0] D_a^2 [I_k \; 0]^T C \quad \text{since diagonal matrices commute} \\
    &= \frac{2}{|A||\mathcal{X}|} \sum_a \left( I - Q_k Q_k^T \right)  Q_k D_a^2 [I_k \; 0]^T C \\
    &= \frac{2}{|A||\mathcal{X}|} \sum_a \left( Q_k - Q_k \right) D_a^2 [I_k \; 0]^T C \\
    &= 0
\end{align*}
}
Thus $\PhiAC^*$ is a critical point.
\end{proof}

\subsection{Proofs of~\Cref{sec:byol-var}: BYOL-VAR}

We already know what $P_a^*$ is from~\Cref{lem:optimal-pa}, and through a similar argument we also know what $P^*$ is. Now we focus on first computing $\dot{\Phi}$ for BYOL-VAR.

\subsubsection{Computing $\dot{\Phi}$}

\begin{restatable}{relemma}{restatelemphidotvar}
\label{lem:phi-dot-var}
Under the uniform policy assumption~\Cref{ass:uniform-policy}. We have the following.
\begin{align*}
    \dot{\Phi} &= - \nabla_{\Phi} \mathbb{E}\left[\| P_a^\top \Phi^\top x - \sg(\Phi^\top y) ) \|_2^2 - \| P^\top \Phi^\top x - \sg(\Phi^\top y) ) \|_2^2\right] \big\vert_{P=P^*, P_a = Pa^*} \\
    &= - \frac{2}{|A|} \sum_a \left( D_X \Phi P_a^* - D_X T_a \Phi \right) (P_a^*)^T + 2 \left( D_X \Phi P^* - D_X T^{\pi} \Phi \right) (P^*)^T
\end{align*}
\end{restatable}

\begin{proof}
We know the following from~\Cref{lem:optimal-pa} (the same steps apply for $P^*$ but we replace $T_a$ with $T^{\pi}$).
\begin{align*}
    \left( \Phi^T D_X \Phi \right) P_a^* = \Phi^T D_X T_a \Phi \\
    \left( \Phi^T D_X \Phi \right) P^* = \Phi^T D_X T^{\pi} \Phi
\end{align*}
Note that because $\pi$ is uniform, we have $T^{\pi} = \frac{1}{|A|}\sum_a T_a$. This also implies $P^* = \frac{1}{|A|}\sum_a P_a^*$.
Then we just apply the steps in~\Cref{lem:phi-dot-ac} on the two terms in $\dot{\Phi}$ to get the difference.
\end{proof}

\subsubsection{Proof of~\Cref{lem:non-collapse-byol-var} and Simplified $P^*$, $P_a^*$ and $\dot{\Phi}$}

\restatelemnoncollapsevar*
\begin{proof}
\begin{align*}
    \Phi^T \dot{\Phi} &= \Phi^T \left( - \frac{2}{|A|} \sum_a \left( D_X \Phi P_a^* - D_X T_a \Phi \right) (P_a^*)^T + 2 \left( D_X \Phi P^* - D_X T^{\pi} \Phi \right) (P^*)^T \right) \\
    &=  - \frac{2}{|A|} \sum_a \left( \underbrace{\Phi^T D_X \Phi P_a^* - \Phi^T D_X T_a \Phi}_{=0\text{ by definition of $P_a^*$}} \right) (P_a^*)^T + 2 \left( \underbrace{\Phi^T D_X \Phi P^* - \Phi^T D_X T^{\pi} \Phi}_{=0\text{ by definition of $P^*$}} \right) (P^*)^T   \\
    &= 0.
\end{align*}
Therefore, $\frac{\mathrm{d}}{\mathrm{d}t}\Phi_t^T \Phi_t = 0$, that is, $\Phi_t^T \Phi_t$ is a constant for all $t$.
Since $\Phi_0^T \Phi_0 = I$ by \cref{ass:orthogonal-init}, the result follows.
\end{proof}

As a consequence, and in combination with~\Cref{ass:orthogonal-init,ass:uniform-state,ass:symmetric-t,ass:uniform-policy,ass:symmetric-ta}, we can simplify our expressions for $P*$, $P_a^*$ and $\dot{\Phi}$. Note that with a uniform distribution for $D_X$, we have $D_X = |\mathcal{X}|^{-1} I$.
\begin{align}
    P* &= \Phi^T T^{\pi} \Phi \label{eq:optimal-p-simplified}\\
    P_a^* &= \Phi^T T_a \Phi \\
    \dot{\Phi} &= 2 \left( I - \Phi \Phi^T \right) \left( \frac{1}{|A||\mathcal{X}|} \sum_a  T_a \Phi \Phi^T T_a \Phi  - T^{\pi} \Phi \Phi^T T^{\pi} \Phi \right) \label{eq:phi-dot-var-simplified}
\end{align}

\subsubsection{Proof of~\Cref{lem:trace-byol-var}}

\restatelemtracevar*
\begin{proof}
To check that $-f_{\text{BYOL-VAR}}$ is a Lyapunov function for the BYOL-VAR ODE, we will verify that its time derivative is strictly negative (it is strictly decreasing) for all non critical points (a critical point being $\Phi$ where $\Phi_t = \Phi \Rightarrow \dot{\Phi} = 0$). By chain rule through trace we have
\begin{align*}
    & \frac{\mathrm{d}}{\mathrm{d}t} \left( -f_{\text{BYOL-VAR}}(\Phi_t) \right) \\
    &= -\mtrace \left( \frac{\partial}{\partial \Phi} \left(\mathrm{Trace} \left( \frac{1}{|A|}\sum_a \left( \Phi^T T_a \Phi \right)^T \Phi^T T_a \Phi - \left( \Phi^T T^{\pi} \Phi \right)^T \Phi^T T^{\pi} \Phi  \right) \right)^T \cdot \dot{\Phi} \right) \\
    &= -\mtrace \left( \mathrm{Trace} \left( \frac{1}{|A|}\sum_a  \Phi^T T_a \Phi \Phi^T T_a -  \Phi^T T^{\pi} \Phi  \Phi^T T^{\pi}  \right) \cdot \dot{\Phi} \right)
\end{align*}
Since $\Phi$ is orthogonal, $(I - \Phi\Phi^T)$ is a projection matrix i.e. $(I - \Phi\Phi^T)(I - \Phi\Phi^T) = (I - \Phi\Phi^T)$. So we can add an extra projection in front of $\dot{\Phi}$
\begin{align*}
     &= -\mtrace \left( \mathrm{Trace} \left( \frac{1}{|A|}\sum_a  \Phi^T T_a \Phi \Phi^T T_a -  \Phi^T T^{\pi} \Phi  \Phi^T T^{\pi}  \right) (I - \Phi\Phi^T) \cdot \dot{\Phi} \right) \\
    &= -\frac{|\mathcal{X}|}{2} \mtrace \left( \dot{\Phi}^T \dot{\Phi} \right)
\end{align*}
Therefore this is strictly negative when $\dot{\Phi} \neq 0$.
\end{proof}

\subsubsection{Proof of~\Cref{thm:ode-byol-var}}

\restatethmodevar*

\begin{proof}

The first thing we do is rewrite the trace objective so that we can apply~\Cref{lem:max-trace} more directly.
\begin{align*}
    & \mathrm{Trace} \left( \frac{1}{|A|}\sum_a  \Phi^T T_a \Phi  \Phi^T T_a \Phi - \Phi^T T^{\pi} \Phi \Phi^T T^{\pi} \Phi  \right) \\
    &= \mathrm{Trace} \left( \E\left[ \Phi^T T_a \Phi  \Phi^T T_a \Phi \right] - \E [\Phi^T T_a \Phi]  \E[\Phi^T T_a \Phi]  \right)
\end{align*}
We first rewrite the sums and $T^{\pi}$ as (pointwise) expectations over the uniform action. Notice how we now have the difference of the expectation of a square with the square of the expectation. This means we can re-express this as a (pointwise) variance term.
\begin{align*}
    & \mathrm{Trace} \left( \E\left[ \Phi^T T_a \Phi  \Phi^T T_a \Phi \right] - \E [\Phi^T T_a \Phi]  \E[\Phi^T T_a \Phi]  \right) \\
    &= \mathrm{Trace} \left( \E\left[ \left( \Phi^T T_a \Phi  - \E[\Phi^T T_a \Phi] \right)^2 \right] \right) \\
    &= \mathrm{Trace} \left( \E\left[ \Phi^T \left( T_a - T^{\pi} \right) \Phi \Phi^T \left( T_a - T^{\pi} \right) \Phi \right] \right) \\
    &= \mathrm{Trace} \left( \frac{1}{|A|}\sum_a \Phi^T \left( T_a - T^{\pi} \right) \Phi \Phi^T \left( T_a - T^{\pi} \right) \Phi \right)
\end{align*}
Now in this new form, we are ready to apply~\Cref{lem:max-trace}.

Let the eigendecomposition of $T_a$ be $T_a = Q D_a Q^T$ for all actions $a$. So $Q$ are the shared eigenvectors across all $T_a$ (\Cref{ass:common-ta}). We also know $T^{\pi} = \frac{1}{|A|} \sum_a T_a$, which means we also have the eigendecomposition $T^{\pi} = Q \left( \frac{1}{|A|} \sum_a D_a\right) Q^T$. In other words, $T^{\pi}$ also shares the same eigenvectors. This means that $ \frac{1}{|A|} \sum_a \left(T_a - T^{\pi}\right)^2$ also has the same eigenvectors.

Let the eigenvectors of $Q$ be sorted from largest to smallest according to $\frac{1}{|A|} \sum_a \left(T_a - T^{\pi}\right)^2$. Let $Q = [Q_k \; \overline{Q}_k]$ so that $Q_k$ are the top-$k$ eigenvectors. 
Let $C \in \real^{k \times k}$ be an arbitrary orthogonal matrix.
From~\Cref{lem:max-trace} we have that $\PhiVARAC^* = Q_k C$ is a maximizer of the trace objective. Because we are mutiplying $Q_k$ by $C$ on the right, $\PhiVARAC^*$ columns span the same subspace as the columns of $Q_k$ i.e. the same subspace as the top-$k$ eigenvectors of $\frac{1}{|A|} \sum_a \left(T_a - T^{\pi}\right)^2$. We also have (variance relationship)
\begin{align*}
    \frac{1}{|A|} \sum_a \left(T_a - T^{\pi}\right)^2 &= \frac{1}{|A|} \sum_a \left(T_a^2 + (T^{\pi})^2 - 2T_a T^{\pi} \right) \\
    &= \frac{1}{|A|} \sum_a T_a^2 + (T^{\pi})^2 - 2T^{\pi} T^{\pi} \\
    &= \frac{1}{|A|} \sum_a T_a^2 - (T^{\pi} )^2
\end{align*}
Thus, equivalently, $\PhiVARAC^*$ columns span the same subspace as the columns of the top-$k$ eigenvectors of $\frac{1}{|A|} \sum_a T_a^2 - (T^{\pi} )^2$, as we wanted to show.

To finish the proof, we show that $\PhiVARAC^*$ is a critical point of the BYOL-VAR ODE, that is, $\Phi_t = \PhiVARAC^* \Rightarrow \dot{\Phi} = 0$. We plug $\PhiVARAC^* = Q_k C$ into $\dot{\Phi}$ (\Cref{eq:phi-dot-var-simplified}).
{\allowdisplaybreaks
\begin{align*}
    & 2 \left( I - \Phi \Phi^T \right) \left( \frac{1}{|A||\mathcal{X}|} \sum_a  T_a \Phi \Phi^T T_a \Phi  - T^{\pi} \Phi \Phi^T T^{\pi} \Phi \right) \\
    &=  2 \left( I - Q_k Q_k^T \right) \left( \frac{1}{|A||\mathcal{X}|} \sum_a  T_a Q_k Q_k^T T_a  - T^{\pi} Q_k Q_k^T T^{\pi} \right)Q_k \\
    &=  2 \left( I - Q_k Q_k^T \right) \left( \frac{1}{|A||\mathcal{X}|} \sum_a  T_a Q_k Q_k^T T_a  - \frac{1}{|\mathcal{X}||A|^2} \sum_{a, a'} T_a Q_k Q_k^T T_{a'} \right)Q_k
\end{align*}
}
To help simplify further we examine more closely the following term.
{\allowdisplaybreaks
\begin{align*}
    & \left( I - Q_k Q_k^T \right) \left( T_{a} Q_k Q_k^T T_{a'} \right)Q_k \\
    &= \left( I - Q_k Q_k^T \right) \left( (Q D_a Q^T) Q_k Q_k^T (Q D_{a'} Q^T) \right)Q_k  \quad \text{substitute $T_a = Q D_a Q^T$}  \\
    &= \left( I - Q_k Q_k^T \right) \left( Q D_a [I_k \; 0]^T [I_k \; 0] D_{a'} Q^T \right) Q_k  \quad \text{where } Q_k^T Q = Q_k^T [Q_k \; \overline{Q}_k] = [I_k \; 0] \\
    &= \left( I - Q_k Q_k^T \right) \left( Q [I_k \; 0]^T [I_k \; 0] D_a D_{a'} Q^T \right) Q_k \quad \text{since diagonal matrices commute} \\
    &= \left( I - Q_k Q_k^T \right) \left( Q_k D_a D_{a'} Q^T \right) Q_k \\
    &= \left( Q_k - Q_k \right) \left(  D_a D_{a'} Q^T \right) Q_k \\
    &= 0
\end{align*}
}
Thus we have $\dot{\Phi} = 0$:
\begin{align*}
    2 \left( I - Q_k Q_k^T \right) \left( \frac{1}{|A||\mathcal{X}|} \sum_a  T_a Q_k Q_k^T T_a  - \frac{1}{|\mathcal{X}||A|^2} \sum_{a, a'} T_a Q_k Q_k^T T_{a'} \right)Q_k &= 0
\end{align*}
Thus $\PhiVARAC^*$ is a critical point.
\end{proof}

\subsection{Proofs of~\Cref{sec:unifying}: Unifying Perspectives}

\subsubsection{Proof of~\Cref{thm:byol-variants-modelbased-equiv}}

\restatethmmodelbased*

\begin{proof}
We start with the proof of \Cref{eq:model-based-byol}. We first expand out the Frobenius norm on the right-hand side. Note that from~\Cref{ass:symmetric-ta} we know that $T^{\pi}$ and $T_a$ are symmetric, and from~\Cref{lem:non-collapse-byol} we know $\Phi$ is orthogonal.
\begin{align*}
    \Vert T^{\pi} - \Phi P \Phi^T \Vert_F &= \mtrace \left( \left( T^{\pi} - \Phi P \Phi^T \right)^T \left( T^{\pi} - \Phi P \Phi^T \right) \right) \\
    &= \mtrace \left( T^{\pi} T^{\pi} - 2 \Phi P^T \Phi^T T^{\pi} + \Phi P^T P \Phi^T \right)
\end{align*}
To minimize w.r.t. $P$, we compute the matrix derivative w.r.t. $P$.
\begin{align*}
    \frac{\partial }{\partial P} \mtrace \left( T^{\pi} T^{\pi} - 2 \Phi P^T \Phi^T T^{\pi} + \Phi P^T P \Phi^T \right) &= -2 \Phi^T T^{\pi} \Phi + 2P
\end{align*}
Then setting this to zero, we solve for the minimizer $P^*$.
\begin{align*}
    0 &= -2 \Phi^T T^{\pi} \Phi + 2P^* \\
    \implies P^* &= \Phi^T T^{\pi} \Phi
\end{align*}
Plugging this back in~\Cref{eq:model-based-byol}, we get
\begin{align*}
    \mtrace \left( T^{\pi} T^{\pi} - 2 \Phi (P^*)^T \Phi^T T^{\pi} + \Phi (P^*)^T P^* \Phi^T \right) &= \mtrace \left( T^{\pi} T^{\pi} \right) - \mtrace \left( \Phi^T T^{\pi} \Phi \Phi^T T^{\pi} \Phi \right) \\
    &= \mtrace \left( T^{\pi} T^{\pi} \right) - f_{\text{BYOL-}\Pi}(\Phi)
\end{align*}
Thus we have~\Cref{eq:model-based-byol} where the constant term is $\mtrace \left( T^{\pi} T^{\pi} \right)$.

Next to prove~\Cref{eq:model-based-byol-ac}, we follow the same steps, except we substitute $T_a$ for $T^{\pi}$. This results in
\begin{align*}
    \frac{1}{|A|} \sum_a \min_{P_a} \; \Vert T_a - \Phi P_a \Phi^T \Vert_F &= \frac{1}{|A|} \sum_a \left( \mtrace \left( T_a T_a \right) - \mtrace \left( \Phi^T T_a \Phi \Phi^T T_a \Phi \right) \right) \\
    &= \left( \frac{1}{|A|} \sum_a \mtrace \left( T_a T_a \right) \right) - f_{\text{BYOL-AC}}(\Phi)
\end{align*}

Finally, we do the same again for~\Cref{eq:model-based-byol-var} but with $(T_a - T^{\pi})$.
\begin{align*}
    & \frac{1}{|A|} \sum_a \min_{P_{\Delta a}} \; \Vert (T_a - T^{\pi}) - \Phi P_{\Delta a} \Phi^T \Vert_F \\
    &= \frac{1}{|A|} \sum_a \left( \mtrace \left( (T_a - T^{\pi})^2 \right) - \mtrace \left( \Phi^T (T_a - T^{\pi}) \Phi \Phi^T (T_a - T^{\pi}) \Phi \right) \right) \\
    &= \left( \frac{1}{|A|} \sum_a \mtrace \left( (T_a - T^{\pi})^2 \right) \right) - \frac{1}{|A|} \sum_a \mtrace \left( \Phi^T T_a \Phi \Phi^T T_a \Phi - \Phi^T T^{\pi} \Phi \Phi^T T^{\pi} \Phi \right) \\
    &= \left( \frac{1}{|A|} \sum_a \mtrace \left( (T_a - T^{\pi})^2 \right) \right) - f_{\text{BYOL-VAR}}(\Phi)
\end{align*}
where the second-last step uses the fact that $T^{\pi} = \frac{1}{|A|} \sum_a T_a$, which holds since $\pi$ is uniform in accordance with \Cref{ass:uniform-policy}.
\end{proof}

\subsubsection{Proof of~\Cref{thm:byol-model-free-equiv}}

\restatethmmodelfree*

\begin{proof}
We start with proving the first equation (\Cref{eq:valuemse-objective}). First, we solve the inner minimization w.r.t. $\theta$, resulting in:
\begin{align*}
    \min_{\theta} \; \Vert T^\pi R - \Phi \theta \Vert^2
\end{align*}
Given its form of a standard linear least squares equation ($\Vert A\theta - B \Vert^2 $), the solution for $\theta$ is:
\begin{align*}
    \theta^* &= \left( \Phi^T \Phi \right)^{-1} \Phi^T T^\pi R \\
    &= \Phi^T T^\pi R \qquad \text{since $\Phi$ is orthogonal}
\end{align*}
Through the same argument except substituting $T^\pi \Phi\Phi^T$ in for $T^\pi$, the solution for $\omega$ is:
\begin{align*}
    \omega^* &= \Phi^T T^\pi \Phi\Phi^T R
\end{align*}
Substituting $\theta^*$ and $\omega^*$ back in~\Cref{eq:valuemse-objective}, and recalling that $ |\mathcal{X}|\E[RR^T] = I$, we get:
{\allowdisplaybreaks
\begin{align*}
    & |\mathcal{X}|\E_R \left[ \Vert T^\pi R - \Phi \theta^* \Vert^2 + \Vert T^\pi \Phi \Phi^T R - \Phi \omega^* \Vert^2 \right] \\
    &= |\mathcal{X}| \E_R \left[  \Vert T^\pi R - \Phi \Phi^T T^\pi R \Vert^2 + \Vert T^\pi \Phi \Phi^T R - \Phi \Phi^T T^\pi \Phi\Phi^T R \Vert^2 \right] \\
    &= |\mathcal{X}| \E_R \left[  \Vert (I- \Phi \Phi^T) T^\pi R \Vert^2 + \Vert (I - \Phi \Phi^T) T^\pi \Phi\Phi^T R \Vert^2 \right] \\
    &= |\mathcal{X}| \E_R \left[  R^T T^\pi (I- \Phi \Phi^T) (I- \Phi \Phi^T) T^\pi R + R^T \Phi\Phi^T T^\pi (I - \Phi \Phi^T)(I- \Phi \Phi^T) T^\pi \Phi\Phi^T R  \right] \\
    &= |\mathcal{X}| \E_R \left[  R^T T^\pi (I- \Phi \Phi^T) T^\pi R + R^T \Phi\Phi^T T^\pi (I - \Phi \Phi^T) T^\pi \Phi\Phi^T R \right] \\
    &= |\mathcal{X}| \E_R \left[ \mtrace\left( R^T T^\pi (I- \Phi \Phi^T) T^\pi R \right) + \mtrace \left( R^T \Phi\Phi^T T^\pi (I - \Phi \Phi^T) T^\pi \Phi\Phi^T R \right) \right] \\
    &=  \mtrace\left( |\mathcal{X}|\E[RR^T] T^\pi (I- \Phi \Phi^T) T^\pi \right) + \mtrace \left( |\mathcal{X}|\E[RR^T] \Phi\Phi^T T^\pi (I - \Phi \Phi^T) T^\pi \Phi\Phi^T \right) \\
    &=  \mtrace\left( T^\pi (I- \Phi \Phi^T) T^\pi \right) + \mtrace \left(\Phi\Phi^T T^\pi (I - \Phi \Phi^T) T^\pi \Phi\Phi^T \right) \\
    &=  \mtrace\left( T^\pi (I- \Phi \Phi^T) T^\pi \right) + \mtrace \left(T^\pi (I - \Phi \Phi^T) T^\pi \Phi\Phi^T \right) \\
    &=  \mtrace\left( T^\pi T^\pi \right) - \mtrace \left(\Phi^T T^\pi \Phi \Phi^T T^\pi \Phi \right) \\
    &=  C - f_{\text{BYOL-}\Pi}(\Phi)
\end{align*}
}
Thus we have proved~\Cref{eq:valuemse-objective}.

Next, for~\Cref{eq:qvaluemse-objective}, the steps are very similar. We first solve the inner minimization for $\theta_a$ and $\omega_a$, which are the same as for $\theta$ and $\omega$ except we substitute $T_a$ for $T^{\pi}$.
\begin{align*}
    \theta_a^* &= \Phi^T T_a R \\
    \omega_a^* &= \Phi^T T_a \Phi\Phi^T R
\end{align*}
Then substituting $\theta_a^{\star}$ and $\omega_a^{\star}$ back in~~\Cref{eq:qvaluemse-objective}, we have.
\begin{align*}
    & |\mathcal{X}| \mathbb{E}_R \left[ \frac{1}{|A|} \sum_a \left( \Vert T_a R - \Phi \theta_a^* \Vert^2 + \Vert T_a \Phi \Phi^T R - \Phi \omega_a^* \Vert^2  \right) \right] \\
    &= \frac{1}{|A|} \sum_a |\mathcal{X}| \mathbb{E}_R \left[ \left( \Vert T_a R - \Phi \theta_a^* \Vert^2 + \Vert T_a \Phi \Phi^T R - \Phi \omega_a^* \Vert^2  \right) \right] \\
    &= \frac{1}{|A|} \sum_a \left[\mtrace\left( T_a T_a \right) - \mtrace \left(\Phi^T T_a \Phi \Phi^T T_a \Phi \right) \right]
\end{align*}
We get the last line above by following the same steps we just did before when substituting in $\theta^*$ and $\omega^*$. Thus
\begin{align*}
    &= \frac{1}{|A|} \sum_a \mtrace\left( T_a T_a \right) - \frac{1}{|A|} \sum_a \mtrace \left(\Phi^T T_a \Phi \Phi^T T_a \Phi \right) \\
    &= C - f_{\text{BYOL-AC}}(\Phi)
\end{align*}
This completes the proof of~\Cref{eq:qvaluemse-objective}.

Finally we prove~\Cref{eq:advantagemse-objective}. We follow the same steps as for the proof for~\Cref{eq:qvaluemse-objective}, except we use $(T_a - T^{\pi})$ in the place of $T_a$. This means we have
\begin{align*}
    & |\mathcal{X}| \mathbb{E}_R \bigg[ \frac{1}{|A|} \sum_a \min_{\theta_a, \omega_a} \; \bigg( \Vert (T_a R - T^{\pi} R) - \Phi \theta \Vert^2 +  \Vert (T_a \Phi \Phi^T R - T^{\pi} \Phi \Phi^T R) - \Phi \omega \Vert^2 \bigg) \bigg] \\
    &= \frac{1}{|A|} \sum_a \mtrace\left( (T_a - T^{\pi}) (T_a - T^{\pi}) \right) - \frac{1}{|A|} \sum_a \mtrace \left(\Phi^T (T_a - T^{\pi}) \Phi \Phi^T (T_a - T^{\pi}) \Phi \right) \\
    &= C - \frac{1}{|A|} \sum_a \mtrace \left(\Phi^T (T_a - T^{\pi}) \Phi \Phi^T (T_a - T^{\pi}) \Phi \right)
\end{align*}
To further simplify this, we note that because we assume a uniform policy (\Cref{ass:uniform-policy}), we have $T^{\pi} = \frac{1}{|A|}\sum_{a} T_{a}$. So expanding it out
\begin{align*}
    & \frac{1}{|A|} \sum_a \mtrace \left(\Phi^T (T_a - T^{\pi}) \Phi \Phi^T (T_a - T^{\pi}) \Phi \right) \\
    &= \frac{1}{|A|} \sum_a \mtrace \left( \Phi^T T_a\Phi \Phi^T T_a \Phi - \Phi^T T_a \Phi \Phi^T T^{\pi} \Phi - \Phi^T T^{\pi} \Phi \Phi^T T_a \Phi + \Phi^T T^{\pi} \Phi \Phi^T T^{\pi} \Phi \right) \\
    &= \frac{1}{|A|} \sum_a \mtrace \left( \Phi^T T_a\Phi \Phi^T T_a \Phi \right) - 2\mtrace \left( \Phi^T T^{\pi} \Phi \Phi^T T^{\pi} \Phi \right) + \mtrace \left( \Phi^T T^{\pi} \Phi \Phi^T T^{\pi} \Phi \right) \\
    &= \frac{1}{|A|} \sum_a \mtrace \left( \Phi^T T_a\Phi \Phi^T T_a \Phi \right) - \mtrace \left( \Phi^T T^{\pi} \Phi \Phi^T T^{\pi} \Phi \right) \\
    &= f_{\text{BYOL-VAR}}(\Phi)
\end{align*}
This completes the proof of~\Cref{eq:advantagemse-objective}.
\end{proof}

\end{document}